\documentclass[10pt, journal]{IEEEtran} 

\IEEEoverridecommandlockouts
\normalsize
\usepackage{cite}
\usepackage{url}
 \usepackage{hyperref} 
 \hypersetup{colorlinks=true, linkcolor=red, citecolor=red, urlcolor=blue} 

\usepackage{listings}
\usepackage{amssymb}
\usepackage{amsmath}
\usepackage{array}
\usepackage{amsthm}
\usepackage{mathptmx}
\usepackage{mathtools, cuted}
\usepackage{algpseudocode}
\usepackage[ruled, vlined, linesnumbered]{algorithm2e}
\usepackage{graphicx}
\usepackage{xcolor}
\usepackage{subcaption}
\usepackage{enumitem}
\allowdisplaybreaks
\newtheorem{Theorem}{Theorem}

\newtheorem{Remark}{Remark}
\newtheorem{Definition}{Definition}

\newcommand{\rs}{\!\!}
\newcolumntype{C}[1]{>{\centering\arraybackslash}p{#1}}
\newcommand{\bblue}{\textcolor{black}}

\newcommand{\Bquad}{\qquad\qquad\qquad\qquad\qquad\qquad\qquad\qquad\qquad}
\newcommand{\Mquad}{\qquad\qquad\qquad\qquad\qquad}
\newcommand{\Squad}{\qquad\qquad\qquad}
\newcommand{\ind}{\emph{Industry $5.0$}}

\usepackage{acronym}
\acrodef{fl}[{\tt FL}]{federated learning}
\acrodef{afl}[{\tt AFL}]{aerial FL}
\acrodef{ml}[ML]{machine learning}
\acrodef{sgd}[{\tt SGD}]{stochastic gradient descent}
\acrodef{cs}[CS]{central server}
\acrodef{2ceoafl}[{\tt 2CEOAFL}]{\underline{c}omputation- and \underline{c}ommunication-\underline{e}fficient \underline{o}nline \underline{a}erial \underline{f}ederated \underline{l}earning}
\acrodef{iot}[IoT] {Internet of Things}
\acrodef{mimo}[MIMO]{multiple-input and multiple-output}

\acrodef{uav}[UAV]{unmanned aerial vehicle}
\acrodef{cv}[CV]{connected vehicle}
\acrodef{acv}[ACV]{aerial connected vehicle}
\acrodef{gbs}[GBS]{ground base station}
\acrodef{bs}[BS]{base station}
\acrodef{sae}[SAE]{Society of Automotive Engineers}
\acrodef{its}[ITS]{intelligent transportation system}
\acrodef{isac}[ISAC]{integrated sensing and communication}
\acrodef{roi}[RoI]{region of interest}
\acrodef{cpu}[CPU]{central processing unit}
\acrodef{gpu}[GPU]{graphics processing unit}
\acrodef{csi}[CSI]{channel state information}
\acrodef{svd}[SVD]{singular value decomposition}
\acrodef{tdd}[TDD]{time division duplexing}
\acrodef{prb}[pRB]{physical resource block}
\acrodef{doa}[DOA]{direction of arrival}
\acrodef{dod}[DOD]{direction of departure}
\acrodef{mpc}[MPC]{multipath component}
\acrodef{fov}[FoV]{field of view}
\acrodef{aoi}[AoI]{age of information}
\acrodef{pdf}[PDF]{probability density function}
\acrodef{gmm}[GMM]{Gaussian mixture model}
\acrodef{fedavg}[FedAvg]{federated averaging}
\acrodef{cdf}[CDF]{cumulative distribution function}

\usepackage[left=1.62cm,right=1.65cm,top=0.76in]{geometry}
\title{Computation- and Communication-Efficient Online FL for Resource-Constrained Aerial Vehicles}	  

\author{
\IEEEauthorblockN{Ferdous Pervej\IEEEauthorrefmark{1}, Richeng Jin\IEEEauthorrefmark{2}, Md Moin Uddin Chowdhury\IEEEauthorrefmark{3}, Simran Singh\IEEEauthorrefmark{4}, {\.I}smail G{\"u}ven{\c{c}}\IEEEauthorrefmark{4}, and Huaiyu Dai\IEEEauthorrefmark{4}} \\ 
\IEEEauthorblockA{\IEEEauthorrefmark{1}Department of Electrical and Computer Engineering, Utah State University, Logan, UT $84322$} \\
\IEEEauthorblockA{\IEEEauthorrefmark{2}Department of Information and Communication Engineering, Zhejiang University, Hangzhou, China 310007} \\
\IEEEauthorblockA{\IEEEauthorrefmark{3}Ericsson Research, Santa Clara, CA, USA 95054 } \\
\IEEEauthorblockA{\IEEEauthorrefmark{4}Department of Electrical and Computer Engineering, NC State University, Raleigh, NC, USA 27695 }\\
\IEEEauthorblockA{Email: {\tt ferdous.pervej@usu.edu} } 
\vspace{-0.25in}
} % <-this % stops a space

\begin{document}
\maketitle
\IEEEpeerreviewmaketitle

\begin{abstract}
Privacy-preserving distributed \ac{ml} and \ac{acv}-assisted\footnote{The term \ac{acv} is used to refer to general aerial vehicles (e.g., \ac{uav}).} edge computing have drawn significant attention lately. 
Since the onboard sensors of \acp{acv} can capture new data as they move along their trajectories, the continual arrival of such `newly' sensed data leads to online learning and demands carefully crafting the trajectories. 
Besides, as typical \acp{acv} are inherently resource-constrained, computation- and communication-efficient \ac{ml} solutions are needed.
Therefore, we propose a \ac{2ceoafl} algorithm to take the benefits of continual sensed data and limited onboard resources of the \acp{acv}. 
In particular, considering independently owned \acp{acv} act as selfish data collectors, we first model their trajectories according to their respective time-varying data distributions.
We then propose a \ac{2ceoafl} algorithm that allows the flying \acp{acv} to (a) prune the received dense \ac{ml} model to make it shallow, (b) train the pruned model, and (c) probabilistically quantize and offload their trained accumulated gradients to the \ac{cs}. 
Our extensive simulation results show that the proposed \ac{2ceoafl} algorithm delivers comparable performances to its non-pruned and non-quantized, hence, computation- and communication-inefficient counterparts.
\end{abstract}


% \vspace{-0.3in}
\acresetall 
\section{Introduction} 
\label{Intro}
% \noindent
% \textbf{Paragraph-1}:
% \begin{itemize}
%     \item Dream of industry 5.0, next generation networks will be critical
%     \item \ac{acv} assisted communication can bring enormous possibilities and take us closer to achieving industry 5.0 \cite{alsamhi2022computing}
%     \item Can do simultaneous sensing and navigation, help edge computing; in many mission-critical scenarios e.g. ITS \cite{menouar2017uav}, remote sensing \cite{zhang2024uav}, ISAC \cite{meng2024uav} etc.
% \end{itemize}
% \vspace{0.15\textheight}
\noindent
Next-generation wireless networks will play pivotal roles in \ind, which requires frequent human-robot interactions.
While \ac{acv}-assisted communications have already drawn significant attention in the current $5$G networks, the \acp{acv} are not limited to providing communication services only: they can be used for many day-to-day tasks, such as \ac{its} \cite{menouar2017uav}, remote sensing \cite{zhang2024uav}, \ac{isac} \cite{meng2024uav}, edge computing \cite{ning2023mobile}, etc., to name a few. 
As such, \acp{acv}' enormous potential in next-generation wireless networks can be leveraged to get one step closer to  \ind.% \cite{alsamhi2022computing}. 

% \noindent
% \textbf{Paragraph-2}:
% \begin{itemize}
%     \item Industry 5.0 will require swift decision making which requires fast computation
%     \item Relate with ML
%     \item Unfortunately, \ac{acv} has limited computation, storage and battery power
%     \item cannot offload due to massive data size, which causes massive communication overheads, may not compute bulky tasks timely; privacy concerns \cite{pervej2023Resource}
% \end{itemize}
% \vspace{0.15\textheight}

Fast computation is mandatory for swift decision-making in \ind~to avoid putting tasks/missions in danger.
However, typical decision-making problems are often complex, combinatorial, and non-convex.
As such, \acp{acv} and \ac{ml} need to go hand in hand with next-generation wireless networks for fast computation and ubiquitous connectivity for \ind. 
Nonetheless, since \acp{acv} are inherently resource-constrained, taking additional benefits from their onboard resources is significantly challenging. 
On the one hand, they have limited (a) battery power, which gives only a limited flight time; (b) storage, which may not allow training/storing bulky \ac{ml} models/datasets; and (c) computation power, which may not allow fast computation.
On the other hand, if powerful \acp{cpu} and/or \acp{gpu} are integrated into the \acp{acv}, computational energy expenses can be humongous, as energy expense is typically a function of the square of the \ac{cpu} cycle frequency \cite{pervej2023Resource}, which can drain their batteries drastically, truncating their flight times. 
Therefore, lightweight distributed solutions are needed to benefit from the \acp{acv}' onboard resources.

% \noindent
% \textbf{Paragraph-3}:
% \begin{itemize}
%     \item While \ac{acv} as a communication helper and/or acting as a server (see \cite{pham2022energy,fu2022federated,zhang2025latency} and the references therein) is frequently studied, benefit of using \acp{acv} as sensing tool and use of new sensed data for \ac{ml} is relatively unexplored
%     \item \acp{acv} bring enormous potential with simultaneous sensing and computing; can help future 5.0
%     \item Privacy-preserving \ac{ml} like \ac{fl} \cite{mcmahan2017communication} is needed
%     \item Unlike \ac{fl} with stationary data, the nature is online
% \end{itemize}
% \vspace{0.1\textheight}

While the use of \acp{acv} as flying \acp{bs} and/or as servers are frequently studied (see \cite{fu2022federated,zhang2025latency} and the references therein), the benefits of \acp{acv}' continual sensed data from their onboard sensors for different \ac{ml} tasks are relatively unexplored. 
Although offloading newly sensed data---forgoing data privacy concerns---may be considered an option in some instances, this comes with hefty communication overheads, overwhelming the limited resources. % of both the \acp{acv} and the underlying network. 
Besides, when the \acp{acv} are owned and controlled by different independent parties, data privacy concerns become real, necessitating privacy-preserving \ac{ml} solutions such as \ac{fl} \cite{mcmahan2017communication}. 
However, unlike traditional (offline) \ac{fl}, the problem nature is online, as data distributions change over time due to continual arrival of newly sensed data \cite{pervej2024online}. 
Therefore, \acp{acv}'s trajectories need to be modeled diligently to enable efficient data sensing considering \emph{\ac{aoi}}, and a new efficient \emph{online \ac{afl}} solution is needed to cope with the limited resources. % of the \acp{acv}. 
\bblue{Both military and general applications (e.g., surveillance, reconnaissance, object detection, electronic warfare, disaster response, search \& rescue operations, traffic control, object detection/classification, demand predictions, traffic flow control, etc.) can benefit from such online \ac{afl}.}

\subsection{State-of-the-Art Aerial Federated Learning Solutions}
\noindent
Many existing works considered \ac{acv}-assisted \ac{fl} with typical (ground) clients distributed over the \ac{roi} wirelessly connected to a flying \ac{uav}, either as the \ac{bs} or the \ac{cs} \cite{fu2022federated,zhang2025latency}.
Some recent studies \cite{zeng2020federated,chintareddy2025federated,wang2022uav} also considered \acp{uav} as flying clients.
In \cite{zeng2020federated}, an \ac{afl} algorithm was proposed considering a leader \ac{uav} acts as the \ac{cs} while some follower \acp{uav} in the same \ac{uav} swarm act as \ac{fl} clients. 
In \cite{chintareddy2025federated}, an online data generation technique is proposed to train a \ac{fl} algorithm for spectrum sensing. 
% More specifically, the \ac{fedavg} \cite{mcmahan2017communication} algorithm is used where the trained model parameters from aerial clients are weighted proportional to their wireless channel conditions during the global model aggregation phase. 
An online hierarchical \ac{afl} algorithm was proposed in \cite{wang2022uav}, in which some \acp{uav} offloaded their data to other \acp{uav} that participated in  training to get personalized models for different \ac{uav} clusters.

% \bigskip 
% \noindent
% Paragraph 1: (General ML without communication/computation efficiency)
% \begin{itemize}
    % \item Zeng \textit{et al.} proposed a UAV-swarm based FL model where the client UAVs compute and communicate with a leader UAV in \cite{zeng2020federated}.
    % \item \cite{chintareddy2025federated} used \acp{uav} for \ac{fl} for spectrum sensing
    % \item \bred{Check this one \cite{wang2022uav} }
% \end{itemize}
% \vspace{0.1\textheight}

% \noindent
% Paragraph 2 (Communication Efficient FL in UAVs):
% \begin{itemize}
    % \item \cite{jing2025air} proposed online \ac{fl}, assuming data arrives continuously under communication packet error; optimized joint client and sample selections
    % \item \cite{wu2024participant} used utility-based (uav) client sampling
    % \item \cite{hou2025split} used split-\ac{fl} using uavs as clients; the benefit of the splitFL depends on the choice of the model and may in fact lead to additional communication in cases \cite[Remark $1$]{pervej2024personalized}.
% \end{itemize}

A handful of studies also considered client scheduling and dividing training tasks to alleviate resource constraints.  
Authors in \cite{jing2025air} proposed an air-ground integrated online \ac{fl} approach, which jointly optimized new sample selection and client scheduling to minimize the training loss under energy and delay constraints. 
A similar idea is also exploited in \cite{wu2024participant}. %: joint sample and client scheduling are optimized to maximize a utility function, which is a function of the loss associated with the clients' training samples and computation speed. 
\cite{hou2025split} proposed a split \ac{fl} algorithm, where only a shallow front-end model block is trained on the \acp{uav} while the remaining (bulky) back-end model blocks are trained on an edge server.
%Although split \ac{fl} is helpful in many cases, the communication overhead largely depends on the model architecture and the choice of the \emph{cut layer} \cite{pervej2024personalized}.

\subsection{Research Gaps and Our Contributions}
\noindent
While some of the above works completely ignored the online nature of the  \ac{afl} \cite{zeng2020federated,hou2025split} and some did consider online \ac{afl} \cite{chintareddy2025federated,wang2022uav,jing2025air,wu2024participant}, these studies did not consider the fact that (a) independently owned \acp{acv} can act selfishly to maximize sensing fresh data, (b) lack of data distribution change over time modeling may hinder efficient trajectory planning, and (c) both computation- and communication-efficient \ac{afl} solution is needed.  
Therefore, we seek new solutions. 
More specifically, our key contributions are summarized below. 
\begin{enumerate}[leftmargin=+0.2in]
    \item[1)] Considering training data is spatially distributed following a \ac{gmm}, we model the class distribution changes using time-varying basis functions and a static projection matrix. % that maps the basis functions to (data) classes. 
    Then, we optimize the independently owned \acp{acv}' trajectories to maximize their collected data with fresh \ac{aoi}.
    \item[2)] We propose a \ac{2ceoafl} algorithm that enables the \acp{acv} to train pruned models, which are computationally efficient, and allow them to probabilistically quantize their trained gradients before offloading to the \ac{cs} to reduce communication overheads further. We then derive the theoretical convergence bound of the proposed algorithm to investigate how various key parameters (e.g., local data distribution change over time, model pruning, gradient quantization, etc.) add additional errors.
    \item[c)] Finally, we validate that the proposed algorithm yields comparable performance when the models are not pruned and gradients are not quantized.
\end{enumerate}

\section{System Model}
\subsection{General System Model}
\noindent
We consider a geographical \ac{roi} %, denoted by $\mathcal{F}$, 
where the \ac{fl} task needs to be performed.
In this \ac{roi}, a \ac{cs} is embedded in a \ac{gbs}. 
Without any loss of generality, let us assume the \ac{gbs} is located at the center of the coordinate system.
In this \ac{roi}, we deploy $\mathcal{U}=\{u\}_{u=0}^{U-1}$ ACVs at a fixed altitude $h_{\mathrm{acv}}$. 
During time $t$, denote the $2$D position of the $u^{\mathrm{th}}$ \ac{acv} by $\mathbf{q}_u^t \coloneqq [x_u^t, y_u^t]^T$.
The \acp{acv} travel in their respective trajectories and participate in \ac{fl} as aerial clients. 
We consider a fully synchronized setting where each global training round $t$ ends within a fixed deadline $\Delta t$. 
Besides, we assume that the communication between the \acp{acv} and the \ac{gbs} is error-free.
% In other words, the \ac{gbs} and \acp{acv} can exchange information without any errors.

\subsection{Data Distribution Model}
\noindent
In this paper, we assume that each \ac{acv} has an initial dataset $\mathcal{D}_u^{t=0}$, which contains some (outdated) data samples.
The dataset is then continually adapted as the \ac{acv} moves across its trajectory.
Suppose that each \ac{acv} uses $\mathcal{D}_u^{t=0}$ to estimate the \emph{spatial} data distribution before the training begins.
More specifically, we assume that data is spatially distributed following independent \ac{gmm} for each \ac{acv}, which has the following density function.
\begin{align}
    p \left(\mathbf{q}_u^t | \pmb{\mu}_c, \pmb{\Lambda}_u\right) = \sum\nolimits_{c=0}^{C-1} \Pi_c \mathcal{N} \left( \mathbf{q}_u^t | \pmb{\mu}_{u,c}, \pmb{\Lambda}_{u,c} \right), 
\end{align}
where $C$ is the total number of clusters/classes, $0 \leq \Pi_c \leq 1$ is the mixture weight that satisfies $\sum_{c=0}^{C-1} \Pi_c=1$, and $\mathcal{N} \left( \mathbf{q}_u^t | \pmb{\mu}_{u,c}, \pmb{\Lambda}_{u,c} \right) = \frac{1}{\left(2\pi \left|\pmb{\Lambda}_{u,c} \right|\right)^{1/2}} \exp \left[\left(\mathbf{q}_u^t - \pmb{\mu}_{u,c}\right)^T \left(\pmb{\Lambda}_{u,c}\right)^{-1} \left(\mathbf{q}_u^t - \pmb{\mu}_{u,c}\right) \right]$ is the \ac{pdf} of the multivariate Gaussian distribution with mean $\pmb{\mu}_{u,c}$ and covariance $\pmb{\Lambda}_{u,c}$.

While the \ac{gmm} above maps the class centers $\pmb{\mu}_u=\{\pmb{\mu}_{u,c}\}_{c=0}^{C-1}$ in our \ac{roi}, the \emph{temporal} distributions of the classes need to be modeled for \emph{online learning}. 
For simplicity, we consider an intuitive general strategy for modeling the temporal label distributions.
Suppose we have a time-dependent basis function $z_u (t) \in \mathbb{R}^K$, where $K>0$ is the latent dimension and a class-to-basis mapping matrix $\mathbf{M}_u \in \mathbb{R}^{C \times K}$. 
We, thus, model the temporal evolution as $\tilde{\pmb{\psi}}_u^t = \mathbf{M}_u z_u (t)$.
Then, we normalize this to get the time-varying class distribution as
\begin{align}
    \psi_{u,c}^t = \exp\left(\tilde{\psi}_{u,c}^t\right) / \left[\sum\nolimits_{c'=0}^{C-1} \exp\left(\tilde{\psi}_{u,c'}^t\right)\right].
\end{align}
% For simplicity, we assume that each \ac{acv} knows their Gaussian mixture model parameters and the time-varying class/label distributions\footnote{We assume this is known beforehand. In practice, these parameters can be estimated, which is not the main focus of this paper.}. 

\subsection{Online Training Dataset Acquisition}
\noindent
We assume that the \acp{acv} have onboard cameras\footnote{Other sensors, such as LiDar, RF sensors, radar, etc., can be used as well.} with fixed \acp{fov}.
Denote the \ac{fov} of the $u^{\mathrm{th}}$ \ac{acv} by $\theta_u$.
When the \ac{acv} is at a location $\mathbf{q}_u^{t}$, it uses the onboard camera to capture images of the ground surface.
Using the property of \emph{isosceles triangles}, we calculate this area as $\left(2 \cdot h_{\mathrm{acv}} \cdot \tan(\theta_u) \right)^2$ {\tt meter}$^2$.
Suppose the entire image captured at \ac{acv}'s location $\mathbf{q}_u^k$ is processed to create $N_u \left ( \mathbf{q}_u^t \right) \in \mathbb{Z}_0^+$ \emph{new} training samples\footnote{This shall depend on the application (e.g, an entire image can be one training sample for applications like object detection, whereas we may divide it into multiple chunks for image classification.)}.
We stress that some information may overlap between two consecutive footprints captured at locations $\mathbf{q}_u^t$ and $\mathbf{q}_u^{t+1}$.
However, depending on $\Delta t$ and data arrival densities, even if $\mathbf{q}_u^{t+1}$ is identical to $\mathbf{q}_u^{t}$, the \ac{acv} may still capture new samples with relatively fresh \ac{aoi}.

We now focus on modeling the trajectories of the \acp{acv} to facilitate the data collection process.  
It is worth noting that we are interested in modeling \acp{acv}' positions during each \ac{fl} rounds.
In other words, the \acp{acv}' continuous maneuvers using kinematics dynamics are not the key focus of this paper.

\subsubsection{Trajectory Optimization}
Let us denote the cluster association using the following indicator function:
\begin{align}
\label{clsAssocIndiFunc}
    \mathbb{I}_{u,c}^t \coloneqq 
    \begin{cases}
        1, & \text{if } \left\Vert \mathbf{q}_u^t - \pmb{\mu}_{u,c} \right\Vert \leq \zeta \lambda_{u,c} \\
        0, & \text{otherwise} \\
    \end{cases},
\end{align}
where $\lambda_{u,c}$ is the diagonal element of $\pmb{\Lambda}_{u,c}$, $0 < \zeta \leq 1$ is a hyperparameter, and $\left\Vert \cdot \right\Vert$ is the $L_2$ norm.  

Each \ac{acv} aims to travel through clusters according to the clusters' time-varying distributions $\pmb{\psi}_{u}^t \coloneqq \left\{\psi_{u,c}^t\right\}_{c=0}^{C-1}$. 
Intuitively, $\psi_{u,c}^t$ works as the class priority since the \ac{acv} has to travel and sense new training data. 
As such, each \ac{acv} wants to solve the following optimization problem independently:
\begin{subequations}
\label{trajetoryOptim4FL}
\begin{align}
    & \underset{\mathbf{q}_u^t, \mathbb{I}_{u,c}^t}{\mathrm{maximize}} \qquad \sum\nolimits_{c=0}^{C-1} \log \left[\sum\nolimits_{t=0}^{T-1} \mathbb{I}_{u,c}^t \cdot \psi_{u,c}^t + \epsilon_{\mathrm{tol}} \right] , \tag{\ref{trajetoryOptim4FL}} \\
    & \mathrm{subject~ to} \qquad \mathrm{C}_1: ~ \mathbb{I}_{u,c}^t \in \{0,1\} , \quad \forall c, t, \\
    &\qquad \mathrm{C}_2: ~ \left\Vert \mathbf{q}_u^t - \pmb{\mu}_{u,c} \right\Vert - \zeta \lambda_{u,c} \leq M \left(1 - \mathbb{I}_{u,c}^t \right), ~ \forall c,t, \\
    &\qquad \mathrm{C}_3: ~ \left\Vert \mathbf{q}_u^t - \pmb{\mu}_{u,c} \right\Vert - \zeta \lambda_{u,c} \geq -M ~ \mathbb{I}_{u,c}^t, \quad \forall c,t, \\
    & \qquad \mathrm{C}_4: ~ \sum\nolimits_{c=0}^{C-1} \mathbb{I}_{u,c}^t \leq 1, \quad \forall t, \\
    &\rs\rs \mathrm{C}_5: \sum\nolimits_{t=t_1}^{t_1+C} \mathbb{I}_{u,c}^t \leq \mathrm{I}_{u,c,\mathrm{max}}^{t_1}, ~ \forall t_1 \coloneqq t \ni [(t+1) ~\mathrm{mod}~ C = 0], \rs \rs \\
    &\qquad \mathrm{C}_6: ~ \left\Vert \mathbf{q}_{u}^{t+1} - \mathbf{q}_{u}^t \right\Vert \geq \mathrm{d_{min}}, \qquad \forall t,
\end{align}
\end{subequations}
where $\epsilon_{\mathrm{tol}} \ll 1$ is a small number added for numerical stability\bblue{, and the \emph{logarithmic} objective function ensures fair cluster associations.}
Constraints $\mathrm{C}_2$ and $\mathrm{C}_3$, where $M > \zeta \lambda_{u,c}$ is a big number, are used to replace the \emph{if-else} conditions in (\ref{clsAssocIndiFunc}) since such conditions usually are not directly implementable in existing solvers.
Intuitively, constraint $\mathrm{C}_2$ has no effect when $\mathbb{I}_{u,c}^t=0$, while constraint $\mathrm{C}_3$ becomes ineffective when $\mathbb{I}_{u,c}^t=1$. 
Besides, constraint $\mathrm{C}_4$ ensures that the \ac{acv} is associated with at max one cluster, while $\mathrm{C}_5$ restricts the \ac{acv} from visiting the same cluster for more than $\mathrm{I}_{u,c,\mathrm{max}}^{t_1}$ times within $C$ consecutive \ac{fl} rounds.
Finally, constraint $\mathrm{C}_6$ is to ensure that the consecutive trajectory points between two \ac{fl} rounds are $\mathrm{d_{min}}$ meters apart.

Unfortunately, (\ref{trajetoryOptim4FL}) is non-convex due to constraints $\mathrm{C}_3$ and $\mathrm{C}_6$, and cannot be directly solved efficiently. 
Suppose initial points $\mathbf{q}_u^{t,j}$'s are given, which let us do first-order approximations as 
\begin{align*}
    &h_1(\mathbf{q}_u^t) \coloneqq \left\Vert \mathbf{q}_u^t - \pmb{\mu}_{u,c} \right\Vert \approx \left\Vert \mathbf{q}_u^{t,j} - \pmb{\mu}_{u,c} \right\Vert + \nonumber\\
    &\qquad \rs \left[\left(\mathbf{q}_u^{t,j} - \pmb{\mu}_{u,c} \right) / \left( \left\Vert \mathbf{q}_u^{t,j} - \pmb{\mu}_{u,c} \right\Vert + \epsilon_{\mathrm{tol}} \right) \right]^T  \left[ \mathbf{q}_u^t - \mathbf{q}_u^{t,j} \right].\\
    &h_2 (\mathbf{q}_u^t, \mathbf{q}_u^{t+1}) \coloneqq \left\Vert \Delta_{\mathbf{q}}^t \right\Vert \approx \left\Vert \Delta_{\mathbf{q}}^{t,j}\right\Vert + \left[\Delta_{\mathbf{q}}^{t,j} / \left\Vert \Delta_{\mathbf{q}}^{t,j} \right\Vert \right]^T  \left[\Delta_{\mathbf{q}}^{t} - \Delta_{\mathbf{q}}^{t,j} \right],
\end{align*}
where $\Delta_{\mathbf{q}}^t \coloneqq \mathbf{q}_{u}^{t+1} - \mathbf{q}_{u}^t $ and $\Delta_{\mathbf{q}}^{t,j} \coloneqq \mathbf{q}_{u}^{t+1,j} - \mathbf{q}_{u}^{t,j} $

Thus, we transform the original problem as
\begin{subequations}
\label{trajetoryOptim4FL_Trans}
\begin{align}
    & \underset{\mathbf{q}_u^t, \mathbb{I}_{u,c}^t}{\mathrm{maximize}} \qquad \sum\nolimits_{c=0}^{C-1} \log \left[\sum\nolimits_{t=0}^{T-1} \mathbb{I}_{u,c}^t \cdot \psi_{u,c}^t + \epsilon_{\mathrm{tol}} \right] , \tag{\ref{trajetoryOptim4FL_Trans}} \\
    & \mathrm{subject~ to} \qquad \mathrm{C}_1, \mathrm{C}_2, \mathrm{C}_4, \mathrm{C}_5, \\
    &\qquad \tilde{\mathrm{C}}_3: ~ h_1 (\mathbf{q}_u^t) - \zeta \lambda_{u,c} \geq -M ~ \mathbb{I}_{u,c}^t, \quad \forall t, \\
    &\qquad \tilde{\mathrm{C}}_6: ~ h_2 (\mathbf{q}_u^t, \mathbf{q}_u^{t+1})  \geq \mathrm{d_{min}}. 
\end{align}
\end{subequations}
This transformed problem is now a mixed-integer convex programming problem, which is solvable using existing tools such as CVXPY \cite{diamond2016cvxpy} with optimizers like MOSEK. % for the given initial points $\mathbf{q}_u^{t,j}$.
We, therefore, solve the original problem approximately using an iterative process summarized in Algorithm \ref{iterTrajAlg}.

\begin{algorithm}[!t]
% \small
\fontsize{8}{8}\selectfont
\SetAlgoLined 
\DontPrintSemicolon
\KwIn{Initial points $\{\mathbf{q}_u^{t,0} \}_{t=0}^{T-1}$, total iteration $J$, precision $\varpi$}
\nl{\textbf{Repeat}: \;} 
\Indp { 
$j \gets j+1$ \;
Use $\mathbf{q}_u^{t,j-1}$ to solve (\ref{trajetoryOptim4FL_Trans})\;
Set $\mathbf{q}_{u}^{t,j} \gets \mathbf{q}_u^{t^*}$ \;
}
\Indm \textbf{Until} converge with precision $\varpi$ or $j=J$ \;
\KwOut{Optimized $\left\{\mathbf{q}_u^{t^*} \right\}_{t=0}^{T-1}$ and $\Big\{ \big\{ \mathbb{I}_{u,c}^{t^*} \big\}_{c=0}^{C-1} \Big\}_{t=0}^{T-1}$}
\caption{Iterative Trajectory Optimization}
\label{iterTrajAlg}
\end{algorithm}

\subsubsection{Continual Data Sensing for Online Learning}
\noindent
Given that the optimized $\left\{\mathbf{q}_u^{t^*} \right\}_{t=0}^{T-1}$ and $\Big\{ \big\{ \mathbb{I}_{u,c}^{t^*} \big\}_{c=0}^{C-1} \Big\}_{t=0}^{T-1}$ are known, we now focus on modeling $N_u\left(\mathbf{q}_u^{t^*}\right)$, i.e., the number of unique samples associated to location $\mathbf{q}_u^{t^*}$.
Intuitively, given that the \emph{spatial} data distribution follows \ac{gmm}, each \ac{acv} should capture more samples from a cluster if it is close to the center of that cluster.
Besides, we also need to consider the \emph{temporal} data distribution parameter $\psi_{u,c}^t$.  
As such, we use the following equation for $N_u\left(\mathbf{q}_u^{t^*}\right)$:
\begin{align} 
\label{model_Nuqt}
\rs \rs N_u \big(\mathbf{q}_u^{t^*}\big) \rs \coloneqq \rs \left\lceil \sum\nolimits_{c=0}^{C-1}  N_{\mathrm{max}} \psi_{u,c}^t \cdot \exp \big[- \big\Vert \mathbf{q}_u^{t^*} - \mu_{u,c} \big\Vert \big] \cdot \mathbb{I}_{u,c}^{t^*} \right\rceil,
\end{align}
where $\lceil \cdot \rceil$ is the ceiling operator and $N_{\mathrm{max}}$ is a hyperparameter.

Therefore, each \ac{acv} uses the following equation to update their local dataset continually as \cite{pervej2025resource}
\begin{align}
\label{datasetUpdateRule}
    \mathcal{D}_u^t \coloneqq \mathcal{D}_u^{t-1} \bigcup \left\{ \mathbf{x}_i, y_i \right\}_{i=0}^{N_u \left( \mathbf{q}_u^{t^*}\right) - 1},
\end{align}
where $\mathbf{x}_i$ and $y_i$ are the $i^{\mathrm{th}}$ feature set and corresponding label, respectively. 

\bblue{From (\ref{datasetUpdateRule}) and (\ref{model_Nuqt}), it is easy to see that, based on the optimized trajectory $\mathbf{q}_u^{t^*}$ and cluster association indicator $\mathbb{I}_u^{t^*}$, each \ac{acv} shall have their data samples proportional to their time-varying class distributions $\psi_{u,c}^t$ in each round $t$.
Therefore, the data distributions in a particular global round $t$ depend on the accumulated samples, based on $\{\psi_{u,c}^t\}_{c=0}^{C-1}$, from all previous rounds.
}

\section{Computation- and Communication-Efficient Online Aerial FL Algorithm}

\subsection{Proposed {\tt 2CEOAFL} Algorithm}
\noindent
Denote the global and local model of \ac{acv} u during global round $t$ by $\mathbf{w}^t \in \mathbb{R}^{p}$ and $\mathbf{w}_u^{t,0} \in \mathbb{R}^{p}$, respectively.
Since the \acp{acv} have limited time and resources (e.g., limited battery, computation, transmit power, etc.), we assume that they cannot train the dense model $\mathbf{w}^{t,0}$ repeatedly.
As such, the \acp{acv} first need to find their respective pruned models that are shallower and easier to train.
In this work, we assume that the \acp{acv} follow the lottery ticket hypothesis \cite{frankle2018the} to find their pruned models.
More specifically, each \ac{acv} performs $\rho$ local rounds\footnote{\bblue{Typically, $\rho$ is very small ($\approx 1$) since the \acp{acv} do not have sufficient resources to train the dense model repeatedly.}} of mini-batch \ac{sgd} using the received dense model as 
\begin{equation}
    \mathbf{w}_u^{t, \rho} = \mathbf{w}_u^{t,0} - \Tilde{\eta}^t\sum\nolimits_{\tau=0}^{\rho - 1} g_u\left(\mathbf{w}_u^{t, \tau}|\mathcal{D}_u^t\right),
\end{equation}
where $\mathbf{w}_u^{t, 0} \gets \mathbf{w}^{t}$, $\Tilde{\eta}^t$ is the \emph{local} learning rate during round $t$ and $\mathbb{E}_{\zeta \sim \mathcal{D}_{u}^t} \left[g_u \left(\mathbf{w}_u^{t, \tau}|\mathcal{D}_u^t\right)\right] \coloneqq \nabla f_u \left( \mathbf{w}_u^{t, \tau} | \mathcal{D}_u^t \right)$, where $\mathbb{E}[\cdot]$ is the expectation operator. 
After performing these $\rho$ \ac{sgd} rounds, the \ac{acv} removes $\tilde{p}$-smallest magnitude entries from $\mathbf{w}_u^{t,\rho}$ and find a corresponding binary mask $\mathbf{m}_u^{t} \in \{0,1\}^{p}$.
Then, the \ac{acv} gets its winning ticket as  $\Bar{\mathbf{w}}_u^{t, 0} \coloneqq \mathbf{w}_u^{t, 0} \odot \mathbf{m}_u^{t}$,
where $\odot$ is the element-wise multiplication operator \cite{frankle2018the}.

The \ac{acv} then performs $\kappa > \rho$ local rounds of model training using this pruned model $\Bar{\mathbf{w}}_u^{t, 0}$ to minimize the following local loss function:
\begin{align}
    f_u \left( \bar{\mathbf{w}}_u^{t,0} |\mathcal{D}_u^t \right) \coloneqq \left[1/\left|\mathcal{D}_u^t\right|\right] \sum\nolimits_{(\mathbf{x},y) \in \mathcal{D}_u^t} l \left(\bar{\mathbf{w}}_u^{t,0} |(\mathbf{x},y) \right),
\end{align}
where $l \left(\bar{\mathbf{w}}_u^{t,0} |(\mathbf{x},y) \right)$ is the loss function (e.g., MSE, cross-entropy, etc.).
Thus, the updated local model is written as
\begin{equation}
\begin{aligned}
    \Bar{\mathbf{w}}_u^{t, \kappa} 
    &= \Bar{\mathbf{w}}_u^{t, 0} - \Tilde{\eta}^t\sum\nolimits_{\tau=0}^{\kappa-1} g \left( \Bar{\mathbf{w}}_u^{t, \tau} | \mathcal{D}_u^t \right) \odot \mathbf{m}_u^{t}. 
\end{aligned}
\end{equation}
Denote the model differences $\mathbf{d}_u^t \coloneqq \big(\bar{\mathbf{w}}_u^{t,0} - \Bar{\mathbf{w}}_u^{t, \kappa}\big)/\Tilde{\eta}^t = \sum_{\tau=0}^{\kappa-1} g_u \left( \bar{\mathbf{w}}_u^{t,\tau} | \mathcal{D}_u^t \right) \odot \mathbf{m}_u^t$. 
We assume that some \acp{acv} may not be able to share their entire $\mathbf{d}_u^t$ due to large communication overheads.
As such, we consider a probabilistic model, where each \ac{acv} offloads the following 
\begin{align}
\label{modeDifUpload}
    \Pi_u^t \coloneqq 
    \begin{cases}
        \mathbf{d}_u^t, & \text{w.p.  } \mathfrak{q}_u^t, \\
        Q(\mathbf{d}_u^t), & \text{w.p. } (1-\mathfrak{q}_u^t),
    \end{cases},
\end{align}
where $\mathfrak{q}_u^t$ is the probability of transmitting the raw, i.e., unquantized, model differences and $Q(\cdot)$ is a low-precision quantizer, which is defined below.

\begin{Definition}[\textbf{Low Precision Quantizer}  \cite{alistarh2017qsgd}]
The low precision stochastic quantization operation of any vector $\mathbf{d} \in \mathbb{R}^p$ with $\mathbf{d} \neq \mathbf{0}$ is defined as
\begin{align}
    Q(\mathbf{d}) \coloneqq \Vert \mathbf{d} \Vert_2  \cdot \mathrm{sign}(d_i) \cdot \xi_i(\mathbf{d}, s), \quad i \in [p],
\end{align}
where $\xi_i(\mathbf{d}, s)$ is a random variable that takes the value of $\frac{l+1}{s}$ with probability $\frac{\vert d_i \vert}{\Vert\mathbf{d} \Vert_2}s -l$ and value of $\frac{l}{s}$ otherwise.
Besides, $s$ is the tuning parameter defining the total quantization levels and $l \in \left[0, s\right)$ is an integer such that $\frac{\vert d_i \vert}{\Vert \mathbf{d} \Vert_2} \in \left[\frac{l}{s}, \frac{l+1}{s} \right)$.    
\end{Definition}

%Therefore, the \ac{acv} offloads $\mathbf{d}_u^t$ with probability $\mathfrak{q}_u^t$ and offloads $Q(\mathbf{d}_u^t) \coloneqq Q \Big( \big(\bar{\mathbf{w}}_u^{t,0} - \Bar{\mathbf{w}}_u^{t, \kappa} \big)/\Tilde{\eta}^t \Big) = Q \left( \sum_{\tau=0}^{\kappa-1} g_u \left( \bar{\mathbf{w}}_u^{t,\tau} | \mathcal{D}_u^t \right) \odot \mathbf{m}_u^t \right)$ with probability $(1-\mathfrak{q}_u^t)$.
As such, we calculate \ac{acv}'s uplink wireless payload as
\begin{equation*}
\label{payload_Size}
    \mathrm{s}_u^{t} \left( \Pi_u^t \right) 
    \rs \leq \rs 
    \begin{cases}
        p (1 - \delta_u^t) \left(1 + 32 \right) + p, & \rs  \text{w.p. } \mathfrak{q}_u^t, \\
        p (1 - \delta_u^t) \left(1 + \lceil \log_2(s) \rceil \right) + 32 + p,\rs \rs & \rs \text{w.p. } (1-\mathfrak{q}_u^t),
    \end{cases}\rs. 
\end{equation*}
Note that for both cases, the \ac{acv} sends the binary mask.
Besides, for the first case, the \ac{acv} sends $1$-bit sign and $32$ bits unquantized model differences.
For the second case, the \ac{acv} sends the $1$-bit sign and chosen $l$ of $\xi_i(\mathbf{d},s)$, which takes at max $\lceil \log_2( s ) \rceil$ bits, and the unquantized $\Vert \mathbf{d}_u^t \Vert_2$, which takes $32$ bits.

Upon receiving the updates from the \acp{acv}, the \ac{cs} updates the global model as 
\begin{equation}
\label{aggRule}
\begin{aligned}
    \mathbf{w}^{t+1} 
    &= \mathbf{w}^t - \eta^t \sum\nolimits_{u=0}^{U-1} \alpha_u \Pi_u^t,
\end{aligned}
\end{equation}
where $\eta^t$ is the global learning rate, $0 \leq \alpha_u \leq 1$ and $\sum_{u=0}^{U-1} \alpha_u = 1$.
\bblue{Note that we use $\alpha_u=1/U$ since we assume mini-batch \ac{sgd} and all \acp{acv} randomly sample an equal number of mini-batches from their local datasets\footnote{\bblue{However, other weighting policies can be easily adopted.}}.}
The \ac{cs}, thus, minimizes the following global loss function.
\begin{align}
    f \left(\mathbf{w}^t | \mathcal{D}^t\right) \coloneqq \sum\nolimits_{u=0}^{U-1} \alpha_u f_u \left( \mathbf{w}^t | \mathcal{D}_u^t \right),
\end{align}
where $\mathcal{D}^t \coloneqq \bigcup_{u=0}^{U-1}\mathcal{D}_u^t$.

Due to continual data sensing, unlike traditional \ac{fl} with stationary datasets, \ac{2ceoafl} wants to find a sequence of global optimal models, $\mathbf{w}^{t^{*}}$, $\forall t$, that minimizes the respective round's global loss function.

\subsection{Convergence Analysis of {\tt 2CEOAFL} Algorithm}
\label{conv2ceoaflSubSection}
\noindent
We make the following standard assumptions \cite{jiang2022model,liu2022Joint, stich2018sparsified,reisizadeh2020FedPAQ,pervej2024online,pervej2024hierarchical}.

\noindent
\textbf{Assumption $\mathbf{1}$}(Smoothness).
    The local loss functions are $\beta$-Lipschitz smooth, i.e., for some $\beta > 0$, we have $\Vert \nabla f_u \left(\mathbf{w} | \mathcal{D}_u^t \right) - \nabla f_u \left(\mathbf{w}'| \mathcal{D}_u^t\right) \Vert \leq \beta \Vert \mathbf{w} - \mathbf{w}' \Vert$.

\noindent
\textbf{Assumption $\mathbf{2}$} (Unbiased gradient with bounded variance). 
    The stochastic mini-batch gradient at each \ac{acv} is unbiased, i.e., $\mathbb{E}_{\xi \sim \mathcal{D}_u^t} [g_u \left(\mathbf{w}_u | \mathcal{D}_u^t\right)] = \nabla f_u \left( \mathbf{w}_u | \mathcal{D}_u^t\right)$, for all $u$ and $\mathbf{w}_u$.
    Besides, the variance of the stochastic gradients are bounded, i.e., $\left\Vert g_u \left( \mathbf{w}_u | \mathcal{D}_u^t \right) - \nabla f_u \left( \mathbf{w}_u | \mathcal{D}_u^t \right) \right\Vert^2 \leq \sigma^2$ for some $\sigma \geq 0$, for all $u$. 

\noindent
\textbf{Assumption $\mathbf{3}$} (Bounded gradient dissimilarity).
There exist some finite constants $\rho_1 \geq 1$ and $\rho_2 \geq 0$ such that the local and global gradients have the following relationship 
\begin{align}
    \left\Vert \nabla f_u \left(\mathbf{w} |\mathcal{D}_u^t \right) \right \Vert^2 \leq  \rho_1 \left \Vert \nabla f \left(\mathbf{w} | \mathcal{D}^t \right) \right\Vert^2 + \rho_2 \epsilon_u^t,
\end{align}
where $\epsilon_u^t$ is the difference between the statistical data distributions of $\mathcal{D}^t$ and $\mathcal{D}_u^t$. 
We assume that $\rho_1=1$ and $\rho_2=0$ when all clients have identical data distributions.

\noindent
\textbf{Assumption $\mathbf{4}$}.
The stochastic quantization operation is unbiased, i.e., $\mathbb{E}_Q \left[Q (\mathbf{d}) \right] = \mathbf{d}$. 
Besides, its variance grows as 
\begin{align}
    \mathbb{E}_Q \left[ \left\Vert Q (\mathbf{d}) - \mathbf{d} \right\Vert^2 \right] \leq q \Vert \mathbf{d} \Vert^2 ,
\end{align}
where $q$ is a positive real constant.

\noindent
\textbf{Assumption $\mathbf{5}$} (Pruning ratio \cite{stich2018sparsified}).
The pruning ratio $\delta_u^t \in [0,\delta_{\mathrm{th}}]$, where $0 < \delta_{\mathrm{th}} \leq 1$, is bounded as follows:
\begin{equation}
\label{pruneRatio}
    \delta_u^t \geq \left\Vert \mathbf{w}_u^{t} - \Bar{\mathbf{w}}_u^{t,0} \right\Vert^2 / \left\Vert \mathbf{w}_u^{t} \right\Vert^2. 
\end{equation}
Moreover, $\mathbf{w}_u^{t, \tau} = \Bar{\mathbf{w}}_u^{t, \tau}$ only when $\delta_u^t = 0$.  

\begin{Definition}[Local data distribution shift \cite{pervej2024online}]
Suppose that the $u^{\mathrm{th}}$ \ac{acv} has datasets $\mathcal{D}_u^{t-1}$ and $\mathcal{D}_u^t$ during the global round $(t-1)$ and $t$, respectively. 
Then, there exist a $\Phi_u^t \geq 0$ that measures the shifts in the distributions of the \ac{acv}'s datasets between two consecutive global round as 
\begin{align}
\label{localDataShift_Def}
    \left \Vert \nabla f_u \left(\mathbf{w} | \mathcal{D}_u^{t-1}  \right) - \nabla f_u \left(\mathbf{w} | \mathcal{D}_u^t  \right) \right\Vert^2 \leq \Phi_u^t, \qquad \forall u \in \mathcal{U}, 
\end{align}
with $\Phi_u^{t=0}=0$.
\end{Definition}

Since the loss functions are typically not convex, we seek a sub-optimal convergence bound. 
We define $\frac{1}{T} \sum_{t=0}^{T-1} \mathbb{E} \left\Vert \bar{\nabla} f \left( \mathbf{w}^{t} | \mathcal{D}^t \right) \right\Vert^2 \coloneqq \pmb{\theta}$, where $\bar{\nabla} f (\cdot)  \coloneqq \sum_{u=0}^{U-1} \alpha_u \nabla f_u (\cdot) \odot \mathbf{m}_u^t$, and derive a $\pmb{\theta}$-sub-optimal convergence bound in what follows.
\begin{Theorem}
\label{conv_Theorem}
Suppose the above assumptions hold. 
When the learning rates $\eta^t \leq \frac{1}{\beta\kappa(2+q)}$ and $\Tilde{\eta}^t < \text{min} \left\{\rs\frac{1}{2\sqrt{2}\beta\kappa}, \frac{1}{2\sqrt{2\rho_1}\beta\kappa} \rs \right\}$, the global gradient is upper-bounded as
\begin{align}
    & \mathbb{E} \!\left[\!\left\Vert \bar{\nabla} f \left( \mathbf{w}^{t} | \mathcal{D}^t \right) \right\Vert^2 \! \right] \rs  
    \leq \rs \frac{2 \! \left( \mathbb{E} \left[ f \left(\mathbf{w}^{t} | \mathcal{D}^t \right) \right] - \mathbb{E} \left[ f \left( \mathbf{w}^{t+1} | \mathcal{D}^{t+1} \right) \right] \right)} {\eta^t \kappa} \nonumber\\
    &\quad + \beta \underbrace{\sigma^2}_{\mathrm{error:~SG}} \Big( \underbrace{\kappa \eta^t \sum\nolimits_{u=0}^{U-1} \alpha_u^2 \mathrm{C}_u\left(q,\mathfrak{q}_u^t \right) }_{\mathrm{error:~quantization}} + 2 \beta \kappa (\Tilde{\eta}^t)^2 \Big) + \nonumber\\
    &~~ \underbrace{16 \beta^2 \kappa^2 (\Tilde{\eta}^t)^2 \sum\nolimits_{u=0}^{U-1} \alpha_u \Phi_u^t}_{\mathrm{error: ~local~data~ dist.~shift}} + \underbrace{8 \rho_2 \beta^2 \kappa^2 (\Tilde{\eta}^t)^2 \sum\nolimits_{u=0}^{U-1} \alpha_u \epsilon_u^t}_{\mathrm{error:~gradient ~ dissimilarity}} + \nonumber\\
    &~~\underbrace{2 \beta^2 \sum\nolimits_{u=0}^{U-1} \alpha_u \delta_u^t \mathbb{E} \left[\left \Vert \mathbf{w}_u^t \right\Vert^2 \right]}_{\mathrm{error:~model~ pruning}},  \label{convgBound_Gradient_Evolution}
\end{align}
where $\mathrm{C}_u \left(q,\mathfrak{q}_u^t \right) \coloneqq 2 + 2q +(4+q)\left(\mathfrak{q}_u^t\right)^2 - (3q+4)\mathfrak{q}_u^t $. 
Moreover, averaging over time, we get the convergence bound as
\begin{align}
    &\pmb{\theta}  
    \leq \frac{2}{\kappa T} \sum\nolimits_{t=0}^{T-1} \left[ \! \left( \mathbb{E} \left[ f \left(\mathbf{w}^{t} | \mathcal{D}^t \right) \right] - \mathbb{E} \left[ f \left( \mathbf{w}^{t+1} | \mathcal{D}^{t+1} \right) \right] \right) \!/ \eta^t \right] + \nonumber\\
    & \left(\beta  \sigma^2 \rs / T \right) \rs \left( \sum\nolimits_{u=0}^{U-1} \rs \alpha_u^2 \sum\nolimits_{t=0}^{T-1} \rs \eta^t  \mathrm{C}_u\left(q, \mathfrak{q}_u^t\right) + 2 \beta \kappa \sum\nolimits_{t=0}^{T-1} (\Tilde{\eta}^t)^2 \!\right) \! + \nonumber\\
    & \frac{16 \beta^2 \kappa^2}{T} \! \sum_{t=0}^{T-1} \rs \left(\Tilde{\eta}^t\right)^2 \! \sum_{u=0}^{U-1} \rs  \alpha_u \Phi_u^t + \frac{8 \rho_2 \beta^2 \kappa^2 }{T} \sum_{t=0}^{T-1} \left(\Tilde{\eta}^t\right)^2 \sum_{u=0}^{U-1} \alpha_u \epsilon_u^t + \nonumber\\
    & \left(2 \beta^2 / T \right) \sum\nolimits_{t=0}^{T-1} \sum\nolimits_{u=0}^{U-1} \alpha_u \delta_u^t \mathbb{E} \left[\left \Vert \mathbf{w}_u^t \right\Vert^2 \right]. \label{convBoundAvgOverTime}
\end{align}
\end{Theorem}
\begin{proof}
The proof is left in our online supplementary materials \cite{pervej20252ceoaflSupp} due to page limitations.
\end{proof}

The first term in (\ref{convgBound_Gradient_Evolution}) captures the change in the global loss function between two consecutive rounds, while the second term is the consequence of the mini-batch \ac{sgd} and the gradients quantization.
Besides, the third term is due to continual local data distribution shifts \bblue{due to clients' time-varying data distributions}, while the fourth term reflects the effect of bounded gradient dissimilarities between the \acp{acv} due to statistical data heterogeneity. 
Finally, the last term is the effect of model pruning.

\begin{Remark}
While diminishing learning rates $\eta^t \rightarrow 0$ and $\Tilde{\eta}^t \rightarrow 0$ as $t \rightarrow \infty$, may make the noise from the stochastic gradients, i.e., $\sigma^2$, quantization $q$, local data distribution shift $\Phi_u^t$ and bounded gradient dissimilarity, i.e., $\epsilon_u^t$, terms close to $0$, small learning rates may slow down the learning performance.
Besides, since these error terms get scaled by $\kappa$, one may choose a local rate that satisfies $\eta^t \propto \frac{1}{\kappa}$ to control these error terms. 
However, regardless of the choice of the learning rate, the error from model pruning in the last term does not vanish unless $\bar{\mathbf{w}}_u^{t,0}=\mathbf{w}_u^{t,0}$.
Therefore, the global gradient may only converge to a neighborhood of a stationary point. 
\end{Remark}

% \begin{Remark}
% The above discussions point out that the global gradient norm is largely dependent on the time-varying dataset $\mathcal{D}_u^t$ and the model pruning threshold.
% In particular, the former leads to local data distribution shifts and non-IID distributions across clients also lead to gradient dissimilarity.
% The latter, on the other hand, is controllable by choosing a smaller $\delta_u^t$, although the available (networking and computation) resources can highly influence that decision (e.g., see \cite{pervej2024hierarchical} and the references therein). 
% Due to unknown data distributions across the \ac{roi} $\mathcal{F}$ and time-varying channel states, we cannot minimize the right-hand side of (\ref{convBoundAvgOverTime}).
% \end{Remark}

\subsection{Limitations and Future Works}
\noindent
Evidently, (\ref{convBoundAvgOverTime}) gets scaled by the pruning ratios, which are not optimized. 
Besides, while our focus has been on designing a computation- and communication-efficient online \ac{fl} solution, the impacts of newly arrived samples are not explicitly modeled in this paper.
Furthermore, since the dataset is continually updated, the aggregation weights and pruning ratios---which are also impacted by the other parameters (e.g., radio resources, power allocation, user scheduling, CPU frequency, etc., see \cite{pervej2024hierarchical} and the references therein)---can be jointly optimized.
We will address these limitations in our future work.

\section{Simulation Results and Discussions}
\label{simResults}
\begin{figure*}
\begin{minipage}{0.49\textwidth}
    \centering
    \includegraphics[trim=15 3 20 10, clip, width=\textwidth, height=0.24\textheight]{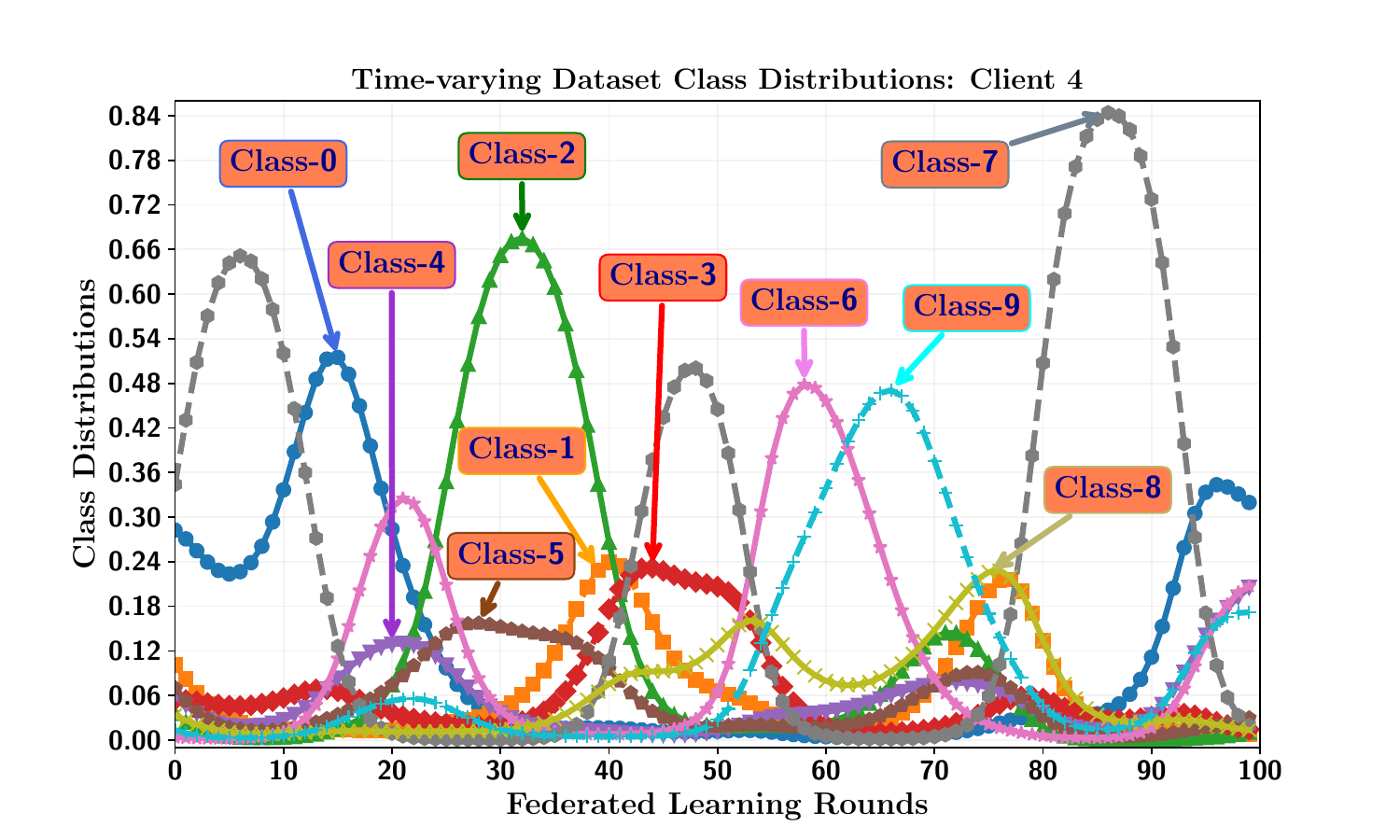}
    \caption{\Ac{acv}'s time-varying data distribution }
    \label{timeVarDataDist}
\end{minipage} \hspace{0.15in}
\begin{minipage}{0.49\textwidth}
    \centering
    \includegraphics[trim=20 5 4 10, clip, width=\textwidth, height=0.24\textheight]{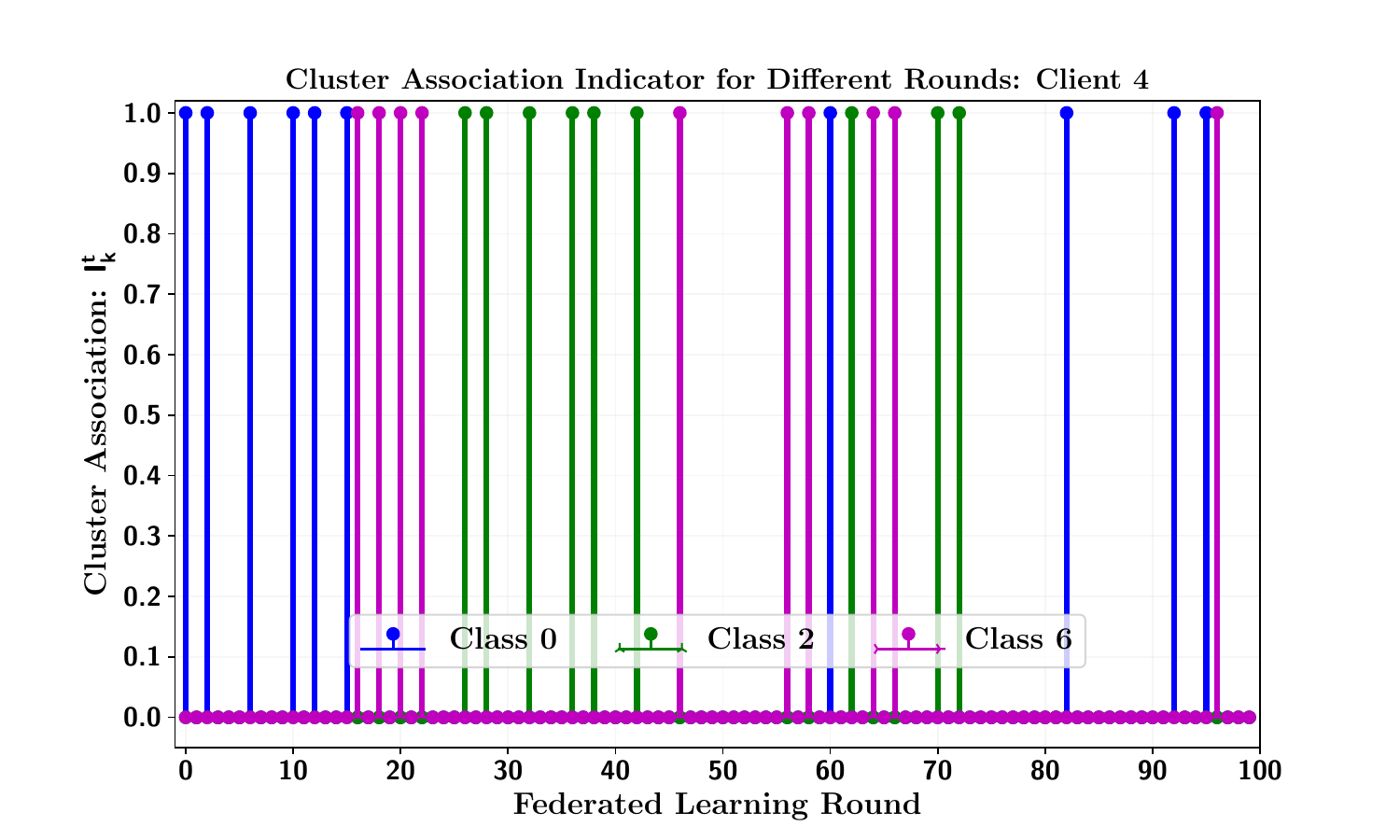}
    \caption{\Ac{acv}'s cluster association}
    \label{clusterAssocIndiResults}
\end{minipage}
\begin{minipage}{0.49\textwidth}
    \centering
    \includegraphics[trim=0 0 10 20, clip, width=\textwidth, height=0.24\textheight]{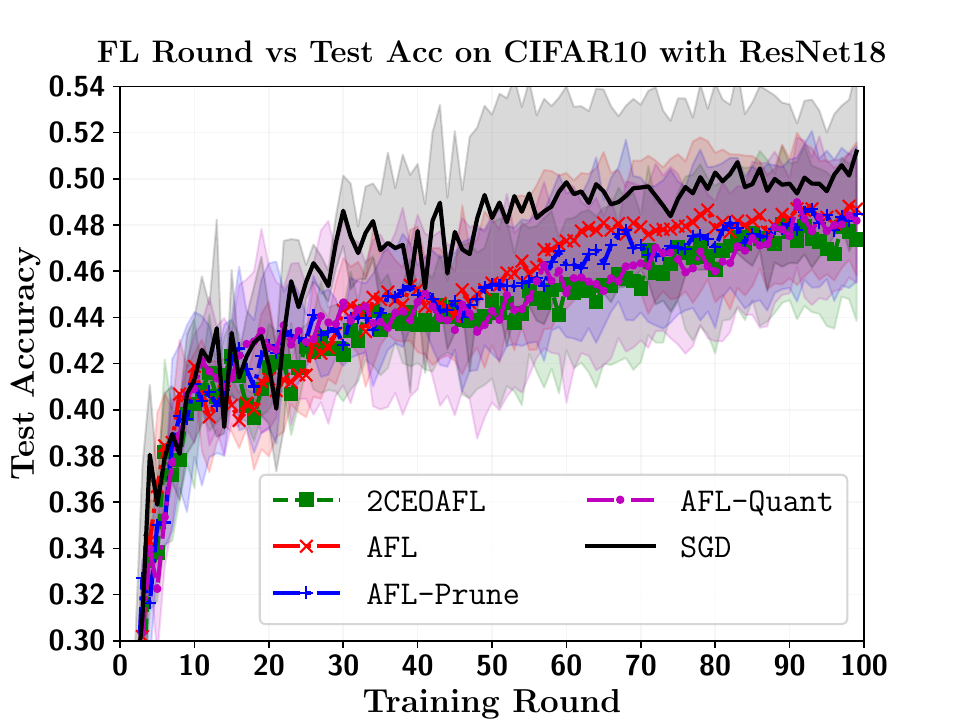}
    \caption{Test accuracies with {\tt ResNet18}: $s=3, \delta_{\mathrm{th}}=0.7$}
    \label{testAccComp}
\end{minipage} \hspace{0.15in}
\begin{minipage}{0.49\textwidth}
    \centering
    \includegraphics[trim=12 0 10 20, clip, width=\textwidth, height=0.24\textheight]{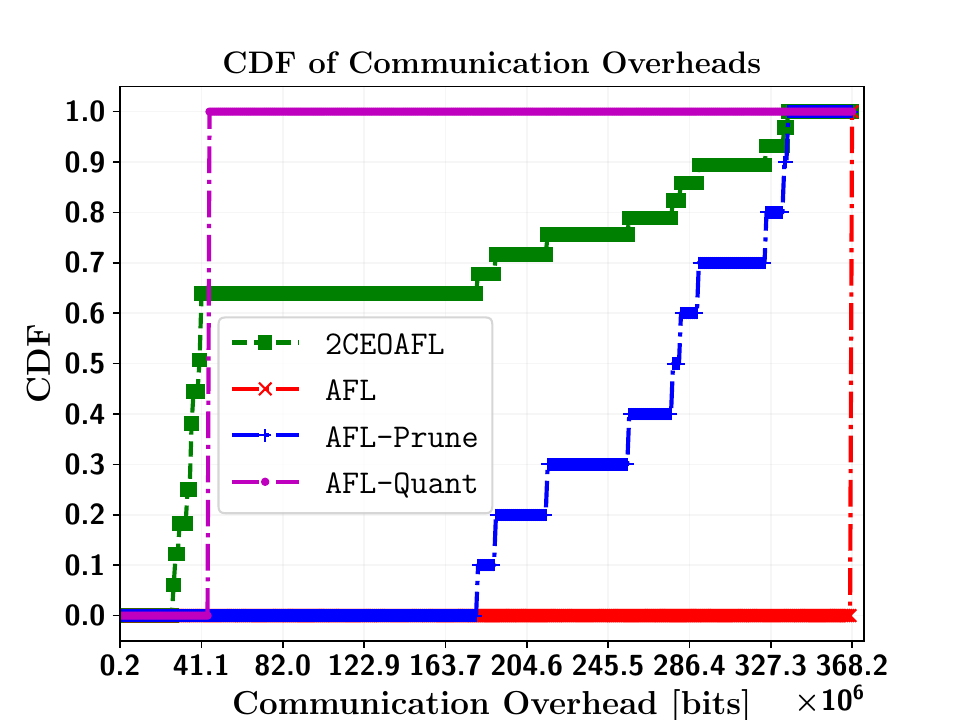}
    \caption{CDF of communication overheads: $s=3, \delta_{\mathrm{th}}=0.7$}
    \label{cdfCommOverheads}
\end{minipage}
\vspace{-0.1in}
\end{figure*}

\subsection{Simulation Setting}
\label{section_sim_set}
\noindent
For proof of concept, we consider an online image classification task using the popular {\tt CIFAR$10$} dataset, which has $C=10$ classes: each class has $10000$ training and $1000$ test samples.
We use $U=10$ \acp{acv} and $T=100$. 
Note that while hundreds of clients are typically used for performance evaluations in \ac{fl}, deploying too many \acp{acv} into a confined \ac{roi} may not be practical. 
Besides, since (readily available) existing datasets have a limited number of training samples, having more \acp{acv} means fewer distinct samples can be allocated to these \acp{acv}. 
For the time-varying basis functions, we use $\big[\sin\big(0:2\pi/T:2\pi \big), \cos\big(0:5\pi/T:5\pi\big), \sin\big(0:5\pi/T:5\pi\big), \cos\big(0:2\pi/T:2\pi\big)\big]^T$ and generate the basis-to-class mapping matrix $\mathbf{M}_u$ randomly. 
Furthermore, for simplicity, we first separate $10000 / U = 1000$ and $1000/U = 100$ unique training and test samples, respectively, to be allocated to the \acp{acv}.
Then, we prepare the initial $\mathcal{D}_u^{t=0}$ by assigning $\lceil \mathbb{E}[\psi_{u,c}^t] \times 512 \rceil$ samples from each $c$.
Moreover, we use $N_{\mathrm{max}}=420$, (\ref{model_Nuqt}) and (\ref{datasetUpdateRule}) to update the datasets in $t>0$.
We stress that the test dataset is also time-varying. 
We prepare it by concatenating $\lceil \mathbb{E}[\psi_{u,c}^t] \times 128 \rceil$ samples from each class and \ac{acv} at $t=0$, and then update it using $N_\mathrm{max}=80$ and (\ref{datasetUpdateRule}) for $t>0$.

A {\tt ResNet18} model is used as our \ac{ml} model, which is trained using \ac{sgd} optimizer with initial $\eta=0.1$ and $\tilde{\eta}=0.1$. 
The learning rates are decreased by $10\%$ in every $20$ and $50$ global rounds, respectively. 
Besides, we use $\kappa=5$, a mini-batch size of $64$, and $5$ mini-batches for local model training.
Moreover, the pruning ratio $\delta_u^t$ is randomly selected from $[0.05, 0.7]$.
Finally, we use $\mathfrak{q}_u^t = \delta_u^t$ since a small $\delta_u^t$ typically leads to large communication overheads.

\subsection{Results and Discussions}
\noindent
\subsubsection{Trajectory Model}
First, let us discuss the impacts of the time-varying class distributions and how that leads to selfish trajectory planning of the \acp{acv}.
Based on our objective function in (\ref{trajetoryOptim4FL_Trans}), intuitively, a \ac{acv} gets the most benefits if it is close to the cluster center with the highest $\psi_{u,c}^t$. 
However, since the problem is non-convex and has many constraints, the trajectory is modeled sub-optimally, maximizing the {\em logarithmic} summation of the $\psi_{u,c}^t$ to ensure fairness among clusters. 
Our simulation results in Figs. \ref{timeVarDataDist} - \ref{clusterAssocIndiResults} also reflect similar trends.
For example, class $2$ has high $\psi_{u=4,c=2}^t$ around $27^{\mathrm{th}}$ to $37^{\mathrm{th}}$ rounds, which leads the \ac{acv} to travel through cluster $2$ around those rounds.
Similar patterns are also observed for the other clusters.

\subsubsection{{\tt 2CEOALF} Results and Baseline Comparisons}
To the best of our knowledge, this is the first work of this kind for \acp{acv} and has no exact baseline. 
For fair comparisons, we consider the following baselines: (1) \ac{sgd}: a centralized baseline assuming a {\em Genie} has access to $\mathcal{D}^t$, for all $t$; (2) \ac{afl}: no pruning or quantization is performed; (3) \ac{afl}-Prune: only model pruning without any quantization; and (4) \ac{afl}-Quant: only gradient quantization without any model pruning.  

Our theoretical analysis in Section \ref{conv2ceoaflSubSection} have shown that both model pruning and quantization contribute to some additional errors to the bound in (\ref{convgBound_Gradient_Evolution}). 
As such, we expect our proposed \ac{2ceoafl} algorithm and the \ac{afl}-Prune and the \ac{afl}-Quant baselines to have slightly lower performance than the \ac{afl} baseline, which does not adopt any model pruning or quantization. 
Besides, since continual data sensing also contributes to local data distribution shifts, the optimal model parameters may also vary across different training rounds, leading to small performance fluctuations.
Moreover, the centralized \ac{sgd} is expected to perform better than the \ac{fl} algorithms since the Genie has access to all \acp{acv}' datasets.

Fig. \ref{testAccComp} also validates the above claims. 
For example, after $T=100$ round, the test accuracies are about $51.5\%$, $48.7\%$, $48.4\%$, $48.2\%$, and $47.5\%$, respectively, with \ac{sgd}, \ac{afl}, \ac{afl}-Prune, \ac{afl}-Quant, and \ac{2ceoafl} algorithms\footnote{\bblue{Note that in our considered system model, while the training data distribution continually changes, the test dataset, which is accumulated from all \acp{acv}' individual test datasets, is also time-varying as described in Section \ref{section_sim_set}. 
This is needed to ensure the trained model is tested on a dataset with a similar time-varying data distribution.}}.
We stress that while the test accuracies are very similar, our proposed solution is both computation- and communication-efficient. 
For example, suppose that one local round of \ac{afl} has $\mathrm{t_{tr}}$ time and $\mathrm{e_{tr}}$ energy overheads.
Then, the local training time overheads are \cite{pervej2024hierarchical} $\kappa \mathrm{t_{tr}}$, $\rho \mathrm{t_{tr}} + \kappa \mathrm{t_{tr}}(1-\delta_u^t)$, $\rho \mathrm{t_{tr}} + \kappa \mathrm{t_{tr}}(1-\delta_u^t)$ and $\kappa \mathrm{t_{tr}}$, while the energy overheads are $\kappa \mathrm{e_{tr}}$, $\rho \mathrm{e_{tr}} + \kappa \mathrm{e_{tr}}(1-\delta_u^t)$, $\rho \mathrm{e_{tr}} + \kappa \mathrm{e_{tr}}(1-\delta_u^t)$ and $\kappa \mathrm{e_{tr}}$, respectively, for \ac{afl}, \ac{2ceoafl}, \ac{afl}-Prune, and \ac{afl}-Quant. 
Since $\rho < \kappa $, typically, $\left[\rho \mathrm{t_{tr}} + \kappa \mathrm{t_{tr}}(1-\delta_u^t) < \kappa \mathrm{t_{tr}} \right]$ and $\rho \mathrm{e_{tr}} + \kappa \mathrm{e_{tr}}(1-\delta_u^t) < \kappa \mathrm{e_{tr}}$.
Moreover, the \ac{afl}-Quant is obviously the most communication-efficient in our setup, as it always performs  gradient quantization before offloading.
However, depending on the resources, we may not always need to quantize and \ac{afl}-Quant is not computation-efficient.
Since our proposed \ac{2ceoafl} algorithm decides whether to quantize the gradients probabilistically, it also exhibits communication efficiency. 
Fig. \ref{cdfCommOverheads} shows the \ac{cdf} of the communication overheads for these \ac{fl} algorithms. 
As we can see, \ac{afl} and \ac{afl}-Prune have the worst and the second-worst communication efficiency.

\section{Conclusions}
% \vspace{0.1 \textheight}
% \begin{itemize}
%     \item Proposed a new way to model spatial and temporal data distributions for enabling continual data sensing with \acp{acv}
%     \item Our proposed \ac{2ceoafl} delivers comparable performances while saving resources
% \end{itemize}
\noindent
This work proposed a new way to model spatial and temporal data distributions for enabling continual data sensing with \acp{acv}.
A \ac{2ceoafl} algorithm was designed, which is both computation- and communication-efficient.
The theoretical analysis showed how model pruning, gradient quantization, and data distribution shifts add errors to the learning performance.
The empirical performance validated that the proposed solution delivers comparable performance by significantly saving computation and communication resources.

\bibliography{references.bib}
\bibliographystyle{IEEEtran}

\newpage	
\begin{appendices}
\onecolumn 
\section*{Supplementary Materials}
\noindent 
\textbf{Additional notations}: $\bar{g}_u \left( \Bar{\mathbf{w}}_u^{t, \tau} | \mathcal{D}_u^t \right) \coloneqq g_u \left( \Bar{\mathbf{w}}_u^{t, \tau} | \mathcal{D}_u^t \right) \odot \mathbf{m}_u^t$, $\bar{\nabla} f_u (\cdot) \coloneqq \nabla f_u (\cdot) \odot \mathbf{m}_u^t$, $\sum_{u=0}^{U-1} \alpha_u \bar{\nabla} f_u (\cdot) \coloneqq \bar{\nabla} f (\cdot) $

\setcounter{Theorem}{0} 
\section{Proof of Theorem \ref{conv_Theorem_Supp}}
\label{appendix_Conv_Theorem}
\noindent
\begin{Theorem}
\label{conv_Theorem_Supp}
Suppose that the above assumptions hold. 
When the learning rates $\eta^t \leq \frac{1}{\beta\kappa(2+q)}$ and $\Tilde{\eta}^t < \text{min} \left\{\frac{1}{2\sqrt{2}\beta\kappa}, \frac{1}{2\sqrt{2\rho_1}\beta\kappa} \right\}$, the global gradient is upper-bounded as
\begin{align}
    \mathbb{E} \left[\left\Vert \bar{\nabla} f \left( \mathbf{w}^{t} | \mathcal{D}^t \right) \right\Vert^2 \right] 
    &\leq \Bigg[ \frac{2 \left( \mathbb{E} \left[ f \left(\mathbf{w}^{t} | \mathcal{D}^t \right) \right] - \mathbb{E} \left[ f \left( \mathbf{w}^{t+1} | \mathcal{D}^{t+1} \right) \right] \right)} {\eta^t\kappa} + \beta \sigma^2 \left( \kappa \eta^t \sum_{u=0}^{U-1} \alpha_u^2 \mathrm{C}_u\left(q,\mathfrak{q}_u^t \right) + 2 \beta \kappa (\Tilde{\eta}^t)^2 \right) + \nonumber\\
    &~ 2 \beta^2 \sum_{u=0}^{U-1} \alpha_u \mathbb{E} \left[\left \Vert \mathbf{w}_u^t - \Bar{\mathbf{w}}_u^{t, 0} \right\Vert^2 \right] + 16 \beta^2 \kappa^2 (\Tilde{\eta}^t)^2 \sum_{u=0}^{U-1} \alpha_u \Phi_u^t + 8 \beta^2 \kappa^2 (\Tilde{\eta}^t)^2 \rho_2 \sum_{u=0}^{U-1} \alpha_u \epsilon_u^t \Bigg], 
\label{convgBound_Gradient_Evolution_Supp}
\end{align}
where $\mathrm{C}_u \left(q,\mathfrak{q}_u^t \right) \coloneqq 2 + 2q +(4+q)\left(\mathfrak{q}_u^t\right)^2 - (3q+4)\mathfrak{q}_u^t $. 
Moreover, averaging over time gives the following $\pmb{\theta}$-suboptimal convergence bound.
\begin{align}
    &\frac{1}{T}\sum_{t=0}^{T-1}  \mathbb{E} \left[\left\Vert \bar{\nabla} f \left( \mathbf{w}^{t} | \mathcal{D}^t \right) \right\Vert^2 \right]  
    \leq \frac{2}{\kappa T} \sum_{t=0}^{T-1} \frac{\left( \mathbb{E} \left[ f \left(\mathbf{w}^{t} | \mathcal{D}^t \right) \right] - \mathbb{E} \left[ f \left( \mathbf{w}^{t+1} | \mathcal{D}^{t+1} \right) \right] \right)} {\eta^t} + \frac{\beta \sigma^2}{T} \left( \sum_{u=0}^{U-1} \alpha_u^2 \sum_{t=0}^{T-1} \eta^t  \mathrm{C}_u\left(q, \mathfrak{q}_u^t\right) + 2 \beta \kappa \sum_{t=0}^{T-1} (\Tilde{\eta}^t)^2 \right) + \nonumber\\
    &\Mquad \frac{16 \beta^2 \kappa^2}{T} \sum_{t=0}^{T-1} \left(\Tilde{\eta}^t\right)^2 \sum_{u=0}^{U-1} \alpha_u \Phi_u^t + \frac{8 \beta^2 \kappa^2 \rho_2}{T} \sum_{t=0}^{T-1} \left(\Tilde{\eta}^t\right)^2 \sum_{u=0}^{U-1} \alpha_u \epsilon_u^t + \frac{2 \beta^2}{T} \sum_{t=0}^{T-1} \sum_{u=0}^{U-1} \alpha_u \mathbb{E} \left[\left \Vert \mathbf{w}_u^t - \Bar{\mathbf{w}}_u^{t, 0} \right\Vert^2 \right]. \label{convBoundAvgOverTime_Supp}
\end{align}
\end{Theorem}

\begin{proof}
Following our aggregation rule in (\ref{aggRule}), we can write
\begin{align}
\label{conv_eqn0}
    f\left(\mathbf{w}^{t+1} | \mathcal{D}^{t+1}\right) 
    &= f \left(\mathbf{w}^t - \eta^t\sum_{u=0}^{U-1} \alpha_u \Pi_u^t \right) \nonumber\\ 
    &\overset{(a)}{\leq} f \left( \mathbf{w}^t|\mathcal{D}^t \right) + \eta^t\bigg< \nabla f \left( \mathbf{w}^t | \mathcal{D}^t \right), -\sum_{u=0}^{U-1} \alpha_u \Pi_u^t \bigg> + \frac{\beta (\eta^t)^2} {2} \left \Vert \sum_{u=0}^{U-1} \alpha_u \Pi_u^t \right \Vert^2,
\end{align}
where $(a)$ stems from $\beta$-Lipschitz smoothness, i.e., $f(y) \leq f(x) + \left< \nabla f(x), y-x \right> + \frac{\beta}{2} \Vert y - x \Vert^2$.

Now, taking expectations on both sides of (\ref{conv_eqn0}), we get the following
\begin{align}
    \mathbb{E} \left[f \left( \mathbf{w}^{t+1} | \mathcal{D}^{t+1} \right) \right] 
    &\leq \mathbb{E} \big[f \left( \mathbf{w}^t|\mathcal{D}^t \right)\big] + \underbrace{\eta^t\mathbb{E} \bigg[\bigg< \nabla f \left( \mathbf{w}^t | \mathcal{D}^t \right), -\sum_{u=0}^{U-1} \alpha_u \Pi_u^t \bigg>\bigg]}_{\mathrm{T}_1} + \underbrace{\frac{\beta (\eta^t)^2} {2} \mathbb{E} \Bigg[ \left \Vert \sum_{u=0}^{U-1} \alpha_u \Pi_u^t \right \Vert^2 \Bigg]}_{\mathrm{T}_2},
    \label{conv_eqn0_1}
\end{align}
where the expectation depends on the randomness of the mini-batch sampling, random quantization events and stochastic quantizer $Q$.

Then, considering all randomness, we simplify the $\mathrm{T}_1$ term of (\ref{conv_eqn0_1}) as
\begin{align}
    &\mathrm{T}_1 = \eta^t\mathbb{E}_{\pmb{\zeta}^t, Q, \mathbf{\mathfrak{q}}^t} \bigg[\bigg< \nabla f \left( \mathbf{w}^t | \mathcal{D}^t \right), - \sum_{u=0}^{U-1} \alpha_u \Pi_u^t \bigg> \bigg] \nonumber\\
    &=-\eta^t\mathbb{E}_{\pmb{\zeta}^t, Q} \bigg[ \mathbb{E}_{\mathbf{\mathfrak{q}}^t |\pmb{\zeta}^t, Q} \bigg[ \bigg< \nabla f \left( \mathbf{w}^t | \mathcal{D}^t \right), \sum_{u=0}^{U-1} \alpha_u \Pi_u^t \bigg>\bigg] \bigg] \nonumber\\
    &=-\eta^t\mathbb{E}_{\pmb{\zeta}^t, Q} \bigg[ \bigg< \nabla f \left( \mathbf{w}^t | \mathcal{D}^t \right), \sum_{u=0}^{U-1} \alpha_u \mathbb{E}_{\mathbf{\mathfrak{q}}^t | \pmb{\zeta}^t, Q} \big[ \Pi_u^t \big] \bigg> \bigg] \nonumber\\
    &\overset{(a)}{=} -\eta^t\mathbb{E}_{\pmb{\zeta}^t, Q} \left[ \bigg< \nabla f \left( \mathbf{w}^t | \mathcal{D}^t \right), \sum_{u=0}^{U-1} \alpha_u \mathfrak{q}_u^t \mathbf{d}_u^t + \sum_{u=0}^{U-1} \alpha_u (1-\mathfrak{q}_u^t) \cdot Q \left( \mathbf{d}_u^t \right) \bigg> \right]\nonumber\\
    &=-\eta^t\mathbb{E}_{\pmb{\zeta}^t} \left[ \mathbb{E}_{Q|\pmb{\zeta}^t} \left[ \bigg< \nabla f \left( \mathbf{w}^t | \mathcal{D}^t \right), \sum_{u=0}^{U-1} \alpha_u \mathfrak{q}_u^t \mathbf{d}_u^t + \sum_{u=0}^{U-1} \alpha_u (1-\mathfrak{q}_u^t) \cdot Q \left( \mathbf{d}_u^t \right) \bigg> \right] \right]\nonumber\\
    &=-\eta^t\mathbb{E}_{\pmb{\zeta}^t} \left[ \bigg< \nabla f \left( \mathbf{w}^t | \mathcal{D}^t \right), \sum_{u=0}^{U-1} \alpha_u \mathfrak{q}_u^t \mathbf{d}_u^t + \sum_{u=0}^{U-1} \alpha_u (1-\mathfrak{q}_u^t) \cdot \mathbb{E}_{Q|\pmb{\zeta}^t} \left[ Q \left( \mathbf{d}_u^t \right) \right] \bigg> \right]\nonumber\\
    &\overset{(b)}{=}-\eta^t\mathbb{E}_{\pmb{\zeta}^t} \left[ \bigg< \nabla f \left( \mathbf{w}^t | \mathcal{D}^t \right), \sum_{u=0}^{U-1} \alpha_u \mathfrak{q}_u^t \mathbf{d}_u^t + \sum_{u=0}^{U-1} \alpha_u (1-\mathfrak{q}_u^t) \cdot \mathbf{d}_u^t \bigg> \right]\nonumber\\
    &= -\eta^t\mathbb{E}_{\pmb{\zeta}^t} \left[ \bigg< \nabla f \left( \mathbf{w}^t | \mathcal{D}^t \right), \sum_{u=0}^{U-1} \alpha_u \mathbf{d}_u^t \bigg> \right]\nonumber\\
    &\overset{(c)}{=} -\eta^t\mathbb{E}_{\pmb{\zeta}^t} \left[ \bigg< \nabla f \left( \mathbf{w}^t | \mathcal{D}^t \right), \sum_{u=0}^{U-1} \alpha_u \sum_{\tau=0}^{\kappa-1} g_u \left( \Bar{\mathbf{w}}_u^{t, \tau} | \mathcal{D}_u^t \right) \odot \mathbf{m}_u^t \bigg> \right] \nonumber\\
    &= -\eta^t\bigg< \nabla f \left( \mathbf{w}^t | \mathcal{D}^t \right), \sum_{u=0}^{U-1} \alpha_u \sum_{\tau=0}^{\kappa-1} \mathbb{E}_{\pmb{\zeta}^t} \left[ g_u \left( \Bar{\mathbf{w}}_u^{t, \tau} | \mathcal{D}_u^t \right) \odot \mathbf{m}_u^t \right] \bigg>  \nonumber\\
    &\overset{(d)}{=} -\eta^t\left< \sum_{u=0}^{U-1} \alpha_u \nabla f_u \left(\mathbf{w}^t | \mathcal{D}_u^t \right), \sum_{u=0}^{U-1} \alpha_u \sum_{\tau=0}^{\kappa-1} \nabla f_u \left( \Bar{\mathbf{w}}_u^{t, \tau} | \mathcal{D}_u^t \right) \odot \mathbf{m}_u^t \right> \nonumber\\
    &= -\eta^t\sum_{u=0}^{U-1} \alpha_u \left< \nabla f_u \left(\mathbf{w}^t | \mathcal{D}_u^t \right) \odot \mathbf{m}_u^t, \sum_{\tau=0}^{\kappa-1} \nabla f_u \left( \Bar{\mathbf{w}}_u^{t, \tau} | \mathcal{D}_u^t \right) \odot \mathbf{m}_u^t \right> \nonumber\\
    &= -\eta^t\sum_{u=0}^{U-1} \alpha_u \left< \bar{\nabla} f_u \left( \mathbf{w}^t | \mathcal{D}^t \right), \sum_{\tau=0}^{\kappa-1} \bar{\nabla} f_u \left( \Bar{\mathbf{w}}_u^{t, \tau} | \mathcal{D}_u^t \right) \right> \nonumber\\
    &=-\eta^t\sum_{u=0}^{U-1} \alpha_u \sum_{\tau=0}^{\kappa-1} \left< \bar{\nabla} f_u \left( \mathbf{w}^t | \mathcal{D}^t \right), \bar{\nabla} f_u \left( \Bar{\mathbf{w}}_u^{t, \tau} | \mathcal{D}_u^t \right) \right> \nonumber\\
    &\overset{(e)}{=} -\frac{\eta^t}{2} \sum_{u=0}^{U-1} \alpha_u \sum_{\tau=0}^{\kappa-1} \left[ \left\Vert \bar{\nabla} f_u \left( \mathbf{w}^t | \mathcal{D}_u^t \right) \right\Vert^2 + \left\Vert \bar{\nabla} f_u \left( \Bar{\mathbf{w}}_u^{t, \tau} | \mathcal{D}_u^t \right) \right\Vert^2 - \left\Vert \bar{\nabla} f_u \left( \mathbf{w}^t | \mathcal{D}_u^t \right) - \bar{\nabla} f_u \left( \Bar{\mathbf{w}}_u^{t, \tau} | \mathcal{D}_u^t \right) \right\Vert^2 \right] \nonumber\\
    &= \frac{\eta^t}{2} \sum_{u=0}^{U-1} \alpha_u \sum_{\tau=0}^{\kappa-1} \left\Vert \bar{\nabla} f_u \left( \mathbf{w}^t | \mathcal{D}_u^t \right) - \bar{\nabla} f_u \left( \Bar{\mathbf{w}}_u^{t, \tau} | \mathcal{D}_u^t \right) \right\Vert^2 - \frac{\eta^t}{2} \sum_{u=0}^{U-1} \alpha_u \sum_{\tau=0}^{\kappa-1} \left\Vert \bar{\nabla} f_u \left( \mathbf{w}^t | \mathcal{D}_u^t \right) \right\Vert^2 - \frac{\eta^t}{2} \sum_{\tau=0}^{\kappa-1} \sum_{u=0}^{U-1} \alpha_u \left\Vert \bar{\nabla} f_u \left( \Bar{\mathbf{w}}_u^{t, \tau} | \mathcal{D}_u^t \right) \right\Vert^2 \nonumber\\
    &\overset{(f)}{\leq} \frac{\eta^t \beta^2}{2} \sum_{u=0}^{U-1} \alpha_u \sum_{\tau=0}^{\kappa-1} \left\Vert \mathbf{w}^t - \Bar{\mathbf{w}}_u^{t, \tau} \right\Vert^2 - \frac{\eta^t\kappa}{2} \left\Vert \sum_{u=0}^{U-1} \alpha_u \bar{\nabla} f_u \left( \mathbf{w}^t | \mathcal{D}_u^t \right) \right\Vert^2 - \frac{\eta^t}{2} \sum_{u=0}^{U-1} \alpha_u \sum_{\tau=0}^{\kappa-1} \left\Vert \bar{\nabla} f_u \left( \Bar{\mathbf{w}}_u^{t, \tau} | \mathcal{D}_u^t \right) \right\Vert^2 \nonumber\\
    &\overset{(g)}{\leq} \frac{\eta^t \beta^2}{2} \sum_{u=0}^{U-1} \alpha_u \sum_{\tau=0}^{\kappa-1} \left\Vert \mathbf{w}^t - \Bar{\mathbf{w}}_u^{t, \tau} \right\Vert^2 - \frac{\eta^t\kappa}{2} \left\Vert \bar{\nabla} f \left( \mathbf{w}^t|\mathcal{D}^t \right) \right\Vert^2 - \frac{\eta^t}{2} \sum_{u=0}^{U-1} \alpha_u \sum_{\tau=0}^{\kappa-1} \left\Vert \bar{\nabla} f_u \left( \Bar{\mathbf{w}}_u^{t, \tau} | \mathcal{D}_u^t \right) \right\Vert^2, \label{second_term_sim}
\end{align}
where $(a)$ stems from the two cases for the \ac{acv}'s model difference upload, which is defined in (\ref{modeDifUpload}).
Besides, $(b)$ comes from the unbiased quantizer assumption, while $(c)$ is obtained using the definition of $\mathbf{d}_u^t$. 
$(d)$ stems from the unbiased stochastic gradient assumption.
Furthermore, $(e)$ arises from the fact that $2\left< \mathbf{x}, \mathbf{y} \right> = \Vert \mathbf{x} \Vert^2 + \Vert \mathbf{y} \Vert^2 - \Vert \mathbf{x} - \mathbf{y} \Vert^2 $ for a real vector space.
Moreover, $(f)$ stems due to the convexity of vector norm and Jensen inequality, i.e., $\Vert \sum_{i=1}^I \alpha_i \mathbf{x}_i\Vert^2 \leq \sum_{i=1}^I \alpha_i \Vert \mathbf{x}_i \Vert^2$ and $\beta$-Lipschitz smoothness property. 
In $(g)$, we use the fact $\bar{\nabla} f \left( \mathbf{w}^t|\mathcal{D}^t \right) \coloneqq \sum_{u=0}^{U-1} \alpha_u \bar{\nabla} f_u \left( \mathbf{w}^t | \mathcal{D}^t \right)$ since $f (\mathbf{w}) \coloneqq \sum_{u=0}^{U-1} f_u (\mathbf{w})$.

Now, we approximate the $\mathrm{T}_2$ term in (\ref{conv_eqn0_1}) considering all randomness as
\begin{align}
    &\mathrm{T}_2 = \frac{\beta (\eta^t)^2} {2} \mathbb{E}_{\pmb{\zeta}^t,Q,\mathbf{\mathfrak{q}}^t} \left[ \left \Vert \sum_{u=0}^{U-1} \alpha_u \Pi_u^t \right \Vert^2 \right]  \nonumber\\
    &=\frac{\beta (\eta^t)^2} {2} \mathbb{E}_{\pmb{\zeta}^t,Q} \left[ \mathbb{E}_{\mathbf{\mathfrak{q}}^t|\pmb{\zeta}^t,Q} \left[ \left \Vert \sum_{u=0}^{U-1} \alpha_u \Pi_u^t \right \Vert^2 \right] \right]  \nonumber\\
    &\overset{(a)}{=}\frac{\beta (\eta^t)^2} {2} \mathbb{E}_{\pmb{\zeta}^t,Q} \left[ \mathbb{E}_{\mathbf{\mathfrak{q}}^t|\pmb{\zeta}^t,Q} \left[ \left \Vert \sum_{u=0}^{U-1} \alpha_u \Pi_u^t - \mathbb{E}_{\mathbf{\mathfrak{q}}^t|\pmb{\zeta}^t,Q} \left[ \sum_{u=0}^{U-1} \alpha_u \Pi_u^t \right] \right \Vert^2 \right] + \left(\mathbb{E}_{\mathbf{\mathfrak{q}}^t|\pmb{\zeta}^t,Q} \left[ \sum_{u=0}^{U-1} \alpha_u \Pi_u^t \right] \right)^2 \right]  \nonumber\\
    &=\frac{\beta (\eta^t)^2} {2} \mathbb{E}_{\pmb{\zeta}^t,Q} \left[ \mathbb{E}_{\mathbf{\mathfrak{q}}^t|\pmb{\zeta}^t,Q} \left[ \left \Vert \sum_{u=0}^{U-1} \alpha_u \left(\Pi_u^t -  \mathbb{E}_{\mathbf{\mathfrak{q}}^t|\pmb{\zeta}^t,Q} \left[\Pi_u^t \right] \right) \right \Vert^2 \right] + \left\Vert \sum_{u=0}^{U-1} \alpha_u \mathbb{E}_{\mathbf{\mathfrak{q}}^t|\pmb{\zeta}^t,Q} \left[\Pi_u^t \right] \right\Vert^2 \right]  \nonumber\\
    &=\frac{\beta (\eta^t)^2} {2} \mathbb{E}_{\pmb{\zeta}^t,Q} \Bigg[ \mathbb{E}_{\mathbf{\mathfrak{q}}^t | \pmb{\zeta}^t,Q} \left[\sum_{u=0}^{U-1} \alpha_u^2  \left \Vert \Pi_u^t -  \mathbb{E}_{\mathbf{\mathfrak{q}}^t | \pmb{\zeta}^t,Q} \left[\Pi_u^t \right] \right \Vert^2 + \sum_{u=0}^{U-1} \alpha_u \left(\Pi_u^t -  \mathbb{E}_{\mathbf{\mathfrak{q}}^t | \pmb{\zeta}^t,Q} \left[\Pi_u^t \right] \right) \times \sum_{u'=0, u'\neq u}^{U-1} \alpha_{u'} \left(\Pi_{u'}^t -  \mathbb{E}_{\mathbf{\mathfrak{q}}^t | \pmb{\zeta}^t,Q} \left[\Pi_{u'}^t \right] \right) \right] + \nonumber\\
    &\Squad \left\Vert \sum_{u=0}^{U-1} \alpha_u \mathbb{E}_{\mathbf{\mathfrak{q}}^t | \pmb{\zeta}^t,Q} \left[\Pi_u^t \right] \right\Vert^2 \Bigg]  \nonumber\\
    &\overset{(b)}{=}\frac{\beta (\eta^t)^2} {2} \mathbb{E}_{\pmb{\zeta}^t,Q} \left[ \sum_{u=0}^{U-1} \alpha_u^2 \mathbb{E}_{\mathbf{\mathfrak{q}}^t | \pmb{\zeta}^t,Q} \left[ \left \Vert \Pi_u^t -  \mathbb{E}_{\mathbf{\mathfrak{q}}^t | \pmb{\zeta}^t,Q} \left[\Pi_u^t \right] \right \Vert^2  \right] + \left\Vert \sum_{u=0}^{U-1} \alpha_u \mathbb{E}_{\mathbf{\mathfrak{q}}^t | \pmb{\zeta}^t,Q} \left[\Pi_u^t \right] \right\Vert^2 \right]  \nonumber\\
    &= \frac{\beta (\eta^t)^2} {2} \mathbb{E}_{\pmb{\zeta}^t,Q} \left[ \sum_{u=0}^{U-1} \alpha_u^2 \mathbb{E}_{\mathbf{\mathfrak{q}}^t | \pmb{\zeta}^t,Q} \left[ \left \Vert \Pi_u^t -  \left[\mathfrak{q}_u^t \mathbf{d}_u^t + (1 - \mathfrak{q}_u^t) Q\left( \mathbf{d}_u^t \right) \right] \right \Vert^2  \right] + \left\Vert \sum_{u=0}^{U-1} \alpha_u  \left[\mathfrak{q}_u^t \mathbf{d}_u^t + (1 - \mathfrak{q}_u^t) Q\left( \mathbf{d}_u^t \right) \right] \right\Vert^2 \right]  \nonumber\\
    &= \frac{\beta (\eta^t)^2} {2} \mathbb{E}_{\pmb{\zeta}^t,Q} \Bigg[ \sum_{u=0}^{U-1} \alpha_u^2 \bigg\{ \mathfrak{q}_u^t \left \Vert \mathbf{d}_u^t -  \left[\mathfrak{q}_u^t \mathbf{d}_u^t + (1 - \mathfrak{q}_u^t) Q\left( \mathbf{d}_u^t \right) \right] \right \Vert^2 + (1-\mathfrak{q}_u^t) \left \Vert Q\left(\mathbf{q}_u^t\right) -  \left[\mathfrak{q}_u^t \mathbf{d}_u^t + (1 - \mathfrak{q}_u^t) Q\left( \mathbf{d}_u^t \right) \right \Vert^2  \right] \bigg\} + \nonumber\\
    &\Squad \left\Vert \sum_{u=0}^{U-1} \alpha_u  \mathfrak{q}_u^t \mathbf{d}_u^t + \sum_{u=0}^{U-1} \alpha_u (1 - \mathfrak{q}_u^t) Q\left( \mathbf{d}_u^t \right)  \right\Vert^2 \Bigg] \nonumber\\
    &\overset{(c)}{\leq} \frac{\beta (\eta^t)^2} {2} \mathbb{E}_{\pmb{\zeta}^t,Q} \Bigg[ \sum_{u=0}^{U-1} \alpha_u^2 \bigg\{ \mathfrak{q}_u^t (1-\mathfrak{q}_u^t) \left \Vert \mathbf{d}_u^t -  Q \left( \mathbf{d}_u^t \right) \right \Vert^2  \bigg\} + 2\left\Vert \sum_{u=0}^{U-1} \alpha_u \mathfrak{q}_u^t \mathbf{d}_u^t \right\Vert^2 + 2 \left\Vert \sum_{u=0}^{U-1} \alpha_u (1 - \mathfrak{q}_u^t) Q\left( \mathbf{d}_u^t \right) \right\Vert^2 \Bigg] \nonumber\\
    &=\frac{\beta (\eta^t)^2} {2} \mathbb{E}_{\pmb{\zeta}^t} \left[ \mathbb{E}_{Q|\pmb{\zeta}^t} \left[ \sum_{u=0}^{U-1} \alpha_u^2 \bigg\{ \mathfrak{q}_u^t (1-\mathfrak{q}_u^t) \left \Vert \mathbf{d}_u^t -  Q \left( \mathbf{d}_u^t \right) \right \Vert^2  \bigg\} + 2\left\Vert \sum_{u=0}^{U-1} \alpha_u \mathfrak{q}_u^t \mathbf{d}_u^t \right\Vert^2 + 2 \left\Vert \sum_{u=0}^{U-1} \alpha_u (1 - \mathfrak{q}_u^t) Q\left( \mathbf{d}_u^t \right) \right\Vert^2 \right] \right] \nonumber\\
    &=\frac{\beta (\eta^t)^2} {2} \mathbb{E}_{\pmb{\zeta}^t} \left[  \sum_{u=0}^{U-1} \alpha_u^2 \mathfrak{q}_u^t (1-\mathfrak{q}_u^t) \mathbb{E}_{Q|\pmb{\zeta}^t} \left[\left \Vert \mathbf{d}_u^t -  Q \left( \mathbf{d}_u^t \right) \right \Vert^2  \right] + 2\left\Vert \sum_{u=0}^{U-1} \alpha_u \mathfrak{q}_u^t \mathbf{d}_u^t \right\Vert^2 + 2 \mathbb{E}_{Q|\pmb{\zeta}^t} \left[ \left\Vert \sum_{u=0}^{U-1} \alpha_u (1 - \mathfrak{q}_u^t) Q\left( \mathbf{d}_u^t \right) \right\Vert^2 \right] \right] \nonumber\\
    &\overset{(d)}{\leq} \frac{\beta (\eta^t)^2} {2} \mathbb{E}_{\pmb{\zeta}^t} \left[  \sum_{u=0}^{U-1} \alpha_u^2 \mathfrak{q}_u^t (1-\mathfrak{q}_u^t) q \left \Vert \mathbf{d}_u^t \right \Vert^2 + 2\left\Vert \sum_{u=0}^{U-1} \alpha_u \mathfrak{q}_u^t \mathbf{d}_u^t \right\Vert^2 + 2 \mathbb{E}_{Q|\pmb{\zeta}^t} \left[ \left\Vert \sum_{u=0}^{U-1} \alpha_u (1 - \mathfrak{q}_u^t) Q\left( \mathbf{d}_u^t \right) \right\Vert^2 \right] \right] \nonumber\\
    &= \frac{\beta (\eta^t)^2} {2} \Bigg[ q \sum_{u=0}^{U-1} \alpha_u^2 \mathfrak{q}_u^t (1-\mathfrak{q}_u^t) \left( \mathbb{E}_{\pmb{\zeta}^t} \left[ \left \Vert \mathbf{d}_u^t - \mathbb{E}_{\pmb{\zeta^t}} \left[\mathbf{d}_u^t\right] \right \Vert^2 \right] + \left(\mathbb{E}_{\pmb{\zeta}^t} \left[\mathbf{d}_u^t\right] \right)^2 \right) + \nonumber\\
    &\Squad 2 \mathbb{E}_{\pmb{\zeta}^t} \left[ \left\Vert \sum_{u=0}^{U-1} \alpha_u \mathfrak{q}_u^t \mathbf{d}_u^t - \mathbb{E}_{\pmb{\zeta}^t} \left[ \sum_{u=0}^{U-1} \alpha_u \mathfrak{q}_u^t \mathbf{d}_u^t \right] \right\Vert^2 \right] + 2 \left(\mathbb{E}_{\pmb{\zeta}^t} \left[ \sum_{u=0}^{U-1} \alpha_u \mathfrak{q}_u^t \mathbf{d}_u^t \right]\right)^2 + \nonumber\\
    &\Squad \mathbb{E}_{\pmb{\zeta}^t} \left\{ 2\mathbb{E}_{Q|\pmb{\zeta}^t} \left[ \left\Vert \sum_{u=0}^{U-1} \alpha_u (1 - \mathfrak{q}_u^t) Q\left( \mathbf{d}_u^t \right) - \mathbb{E}_{Q|\pmb{\zeta}^t} \left[\sum_{u=0}^{U-1} \alpha_u (1 - \mathfrak{q}_u^t) Q\left( \mathbf{d}_u^t \right) \right] \right\Vert^2 \right] + 2\left(\mathbb{E}_{Q|\pmb{\zeta}^t}\left[ \sum_{u=0}^{U-1} \alpha_u (1 - \mathfrak{q}_u^t) Q\left( \mathbf{d}_u^t \right) \right] \right)^2 \right\} \Bigg] \nonumber\\
    &\overset{(e)}{=}\frac{\beta (\eta^t)^2} {2} \Bigg[ q \sum_{u=0}^{U-1} \alpha_u^2 \mathfrak{q}_u^t (1-\mathfrak{q}_u^t) \left( \mathbb{E}_{\pmb{\zeta}^t} \left[\left \Vert \sum_{\tau=0}^{\kappa-1} \bar{g}_u \left(\bar{\mathbf{w}}_u^{t,\tau} | \mathcal{D}_u^t \right) - \mathbb{E}_{\pmb{\zeta^t}} \left[\sum_{\tau=0}^{\kappa-1} \bar{g}_u \left(\bar{\mathbf{w}}_u^{t,\tau} | \mathcal{D}_u^t \right)\right] \right \Vert^2 \right] + \left(\mathbb{E}_{\pmb{\zeta}^t} \left[\sum_{\tau=0}^{\kappa-1} \bar{g}_u \left(\bar{\mathbf{w}}_u^{t,\tau} | \mathcal{D}_u^t \right) \right] \right)^2 \right) + \nonumber\\
    &\Squad 2 \mathbb{E}_{\pmb{\zeta}^t} \left[\left\Vert \sum_{u=0}^{U-1} \alpha_u \mathfrak{q}_u^t \left[ \sum_{\tau=0}^{\kappa-1} \bar{g}_u \left(\bar{\mathbf{w}}_u^{t,\tau} | \mathcal{D}_u^t \right) - \mathbb{E}_{\pmb{\zeta}^t} \left[  \sum_{\tau=0}^{\kappa-1} \bar{g}_u \left(\bar{\mathbf{w}}_u^{t,\tau} | \mathcal{D}_u^t \right) \right] \right] \right\Vert^2 \right] + 2 \left(\mathbb{E}_{\pmb{\zeta}^t} \left[ \sum_{u=0}^{U-1} \alpha_u \mathfrak{q}_u^t \sum_{\tau=0}^{\kappa-1} \bar{g}_u \left(\bar{\mathbf{w}}_u^{t,\tau} | \mathcal{D}_u^t \right) \right]\right)^2 + \nonumber\\
    &\Squad \mathbb{E}_{\pmb{\zeta}^t} \left[ 2 \sum_{u=0}^{U-1} \alpha_u^2 (1 - \mathfrak{q}_u^t)^2 \mathbb{E}_{Q|\pmb{\zeta}^t} \left[ \left\Vert  Q \left( \mathbf{d}_u^t \right) -  \mathbf{d}_u^t \right\Vert^2 \right] + 2\left\Vert \sum_{u=0}^{U-1} \alpha_u (1 - \mathfrak{q}_u^t) \mathbf{d}_u^t \right\Vert^2 \right] \Bigg] \nonumber\\
    &\leq \frac{\beta (\eta^t)^2} {2} \Bigg\{ q \sum_{u=0}^{U-1} \alpha_u^2 \mathfrak{q}_u^t (1-\mathfrak{q}_u^t) \left( \sum_{\tau=0}^{\kappa-1} \mathbb{E}_{\pmb{\zeta}^t} \left[\left \Vert \bar{g}_u \left(\bar{\mathbf{w}}_u^{t,\tau} | \mathcal{D}_u^t \right) - \bar{\nabla} f_u \left(\bar{\mathbf{w}}_u^{t,\tau} | \mathcal{D}_u^t \right) \right \Vert^2 \right] + \left\Vert \sum_{\tau=0}^{\kappa-1} \bar{\nabla} f_u \left(\bar{\mathbf{w}}_u^{t,\tau} | \mathcal{D}_u^t \right) \right\Vert^2 \right) + \nonumber\\
    &\Squad 2 \sum_{u=0}^{U-1} \alpha_u^2 \left(\mathfrak{q}_u^t\right)^2 \sum_{\tau=0}^{\kappa-1} \mathbb{E}_{\pmb{\zeta}^t} \left[ \left\Vert \bar{g}_u \left(\bar{\mathbf{w}}_u^{t,\tau} | \mathcal{D}_u^t \right) - \bar{\nabla} f_u \left(\bar{\mathbf{w}}_u^{t,\tau} | \mathcal{D}_u^t \right) \right\Vert^2 \right] + 2 \left\Vert  \sum_{u=0}^{U-1} \alpha_u \mathfrak{q}_u^t \sum_{\tau=0}^{\kappa-1} \bar{\nabla} f_u \left(\bar{\mathbf{w}}_u^{t,\tau} | \mathcal{D}_u^t \right) \right\Vert^2 + \nonumber\\
    &\Squad 2 \sum_{u=0}^{U-1} \alpha_u^2 (1 - \mathfrak{q}_u^t)^2 q ~ \mathbb{E}_{\pmb{\zeta}^t} \left[\left\Vert \mathbf{d}_u^t \right\Vert^2 \right] + 2 \mathbb{E}_{\pmb{\zeta}^t} \left[\left\Vert \sum_{u=0}^{U-1} \alpha_u (1 - \mathfrak{q}_u^t) \mathbf{d}_u^t \right\Vert^2 \right] \Bigg\} \nonumber\\
    &\overset{(f)}{\leq} \frac{\beta (\eta^t)^2} {2} \Bigg\{ q \sum_{u=0}^{U-1} \alpha_u^2 \mathfrak{q}_u^t (1-\mathfrak{q}_u^t) \left[ \sum_{\tau=0}^{\kappa-1} \sigma^2 + \left\Vert \sum_{\tau=0}^{\kappa-1} \bar{\nabla} f_u \left(\bar{\mathbf{w}}_u^{t,\tau} | \mathcal{D}_u^t \right) \right\Vert^2 \right] +  2 \sum_{u=0}^{U-1} \alpha_u^2 \left(\mathfrak{q}_u^t\right)^2 \sum_{\tau=0}^{\kappa-1} \sigma^2 + 2 \left\Vert  \sum_{u=0}^{U-1} \alpha_u \mathfrak{q}_u^t \sum_{\tau=0}^{\kappa-1} \bar{\nabla} f_u \left(\bar{\mathbf{w}}_u^{t,\tau} | \mathcal{D}_u^t \right) \right\Vert^2 + \nonumber\\
    &\Squad 2 q \sum_{u=0}^{U-1} \alpha_u^2 (1 - \mathfrak{q}_u^t)^2 \left( \sum_{\tau=0}^{\kappa-1} \mathbb{E}_{\pmb{\zeta}^t} \left[\left\Vert \left[\bar{g}_u \left(\bar{\mathbf{w}}_u^{t,\tau} | \mathcal{D}_u^t \right) - \bar{\nabla} f_u \left(\bar{\mathbf{w}}_u^{t,\tau} | \mathcal{D}_u^t \right)\right] \right\Vert^2\right] + \left\Vert \sum_{\tau=0}^{\kappa-1} \bar{\nabla} f_u \left(\bar{\mathbf{w}}_u^{t,\tau} | \mathcal{D}_u^t \right) \right\Vert \right) + \nonumber\\
    &\Squad 2\sum_{u=0}^{U-1} \alpha_u^2 (1 - \mathfrak{q}_u^t)^2 \sum_{\tau=0}^{\kappa-1} \mathbb{E}_{\pmb{\zeta}^t} \left[ \left\Vert \bar{g}_u \left(\bar{\mathbf{w}}_u^{t,\tau} | \mathcal{D}_u^t \right) - \bar{\nabla} f_u \left(\bar{\mathbf{w}}_u^{t,\tau} | \mathcal{D}_u^t \right) \right\Vert^2 \right] + 2\left\Vert \sum_{u=0}^{U-1} \alpha_u (1 - \mathfrak{q}_u^t) \sum_{\tau=0}^{\kappa-1} \bar{\nabla} f_u \left(\bar{\mathbf{w}}_u^{t,\tau} | \mathcal{D}_u^t \right) \right\Vert^2 \Bigg\} \nonumber\\
    &\leq \frac{\beta (\eta^t)^2} {2} \Bigg\{ q \sum_{u=0}^{U-1} \alpha_u^2 \mathfrak{q}_u^t (1-\mathfrak{q}_u^t) \left[ \kappa \sigma^2 + \left\Vert \sum_{\tau=0}^{\kappa-1} \bar{\nabla} f_u \left(\bar{\mathbf{w}}_u^{t,\tau} | \mathcal{D}_u^t \right) \right\Vert^2 \right] + 2 \kappa \sigma^2 \sum_{u=0}^{U-1} \alpha_u^2 \left(\mathfrak{q}_u^t\right)^2 + 2 \left\Vert  \sum_{u=0}^{U-1} \alpha_u \mathfrak{q}_u^t \sum_{\tau=0}^{\kappa-1} \bar{\nabla} f_u \left(\bar{\mathbf{w}}_u^{t,\tau} | \mathcal{D}_u^t \right) \right\Vert^2 + \nonumber\\
    &\Squad 2 q \sum_{u=0}^{U-1} \alpha_u^2 (1 - \mathfrak{q}_u^t)^2 \left[\kappa \sigma^2 + \left\Vert \sum_{\tau=0}^{\kappa-1} \bar{\nabla} f_u \left(\bar{\mathbf{w}}_u^{t,\tau} | \mathcal{D}_u^t \right) \right\Vert \right] + 2 \kappa \sigma^2 \sum_{u=0}^{U-1} \alpha_u^2 (1 - \mathfrak{q}_u^t)^2 + 2\left\Vert \sum_{u=0}^{U-1} \alpha_u (1 - \mathfrak{q}_u^t) \sum_{\tau=0}^{\kappa-1} \bar{\nabla} f_u \left(\bar{\mathbf{w}}_u^{t,\tau} | \mathcal{D}_u^t \right) \right\Vert^2 \Bigg\} \nonumber\\
    &\leq \frac{\beta (\eta^t)^2} {2} \Bigg\{  \kappa \sigma^2 \sum_{u=0}^{U-1} \alpha_u^2 \left[2 + 2q +(4+q)\left(\mathfrak{q}_u^t\right)^2 - (3q+4)\mathfrak{q}_u^t \right] + q \kappa \sum_{u=0}^{U-1} \alpha_u^2 \mathfrak{q}_u^t (1-\mathfrak{q}_u^t) \sum_{\tau=0}^{\kappa-1} \left\Vert \bar{\nabla} f_u \left(\bar{\mathbf{w}}_u^{t,\tau} | \mathcal{D}_u^t \right) \right\Vert^2  + \nonumber\\
    &\Squad 2 \sum_{u=0}^{U-1} \alpha_u \left\Vert \mathfrak{q}_u^t \sum_{\tau=0}^{\kappa-1} \bar{\nabla} f_u \left(\bar{\mathbf{w}}_u^{t,\tau} | \mathcal{D}_u^t \right) \right\Vert^2 + 2 q \kappa \sum_{u=0}^{U-1} \alpha_u^2 (1 - \mathfrak{q}_u^t)^2 \sum_{\tau=0}^{\kappa-1} \left\Vert \bar{\nabla} f_u \left(\bar{\mathbf{w}}_u^{t,\tau} | \mathcal{D}_u^t \right) \right\Vert +  2 \sum_{u=0}^{U-1} \alpha_u \left\Vert (1 - \mathfrak{q}_u^t) \sum_{\tau=0}^{\kappa-1} \bar{\nabla} f_u \left(\bar{\mathbf{w}}_u^{t,\tau} | \mathcal{D}_u^t \right) \right\Vert^2 \Bigg\} \nonumber\\
    &\leq \frac{\beta (\eta^t)^2} {2} \Bigg\{  \kappa \sigma^2 \sum_{u=0}^{U-1} \alpha_u^2 \left[2 + 2q +(4+q)\left(\mathfrak{q}_u^t\right)^2 - (3q+4)\mathfrak{q}_u^t \right] + q \kappa \sum_{u=0}^{U-1} \alpha_u^2 \left[ \mathfrak{q}_u^t (1-\mathfrak{q}_u^t) + (1 - \mathfrak{q}_u^t)^2 \right] \sum_{\tau=0}^{\kappa-1} \left\Vert \bar{\nabla} f_u \left(\bar{\mathbf{w}}_u^{t,\tau} | \mathcal{D}_u^t \right) \right\Vert^2  + \nonumber\\
    &\Squad 2 \kappa \sum_{u=0}^{U-1} \alpha_u \left(\mathfrak{q}_u^t\right)^2 \sum_{\tau=0}^{\kappa-1} \left\Vert \bar{\nabla} f_u \left(\bar{\mathbf{w}}_u^{t,\tau} | \mathcal{D}_u^t \right) \right\Vert^2 + 2 \kappa \sum_{u=0}^{U-1} \alpha_u (1 - \mathfrak{q}_u^t)^2 \sum_{\tau=0}^{\kappa-1} \left\Vert \bar{\nabla} f_u \left(\bar{\mathbf{w}}_u^{t,\tau} | \mathcal{D}_u^t \right) \right\Vert^2 \Bigg\} \nonumber\\
    &=\frac{\beta (\eta^t)^2} {2} \Bigg\{  \kappa \sigma^2 \sum_{u=0}^{U-1} \alpha_u^2 \left[2 + 2q +(4+q)\left(\mathfrak{q}_u^t\right)^2 - (3q+4)\mathfrak{q}_u^t \right] + q \kappa \sum_{u=0}^{U-1} \alpha_u^2 (1-\mathfrak{q}_u^t)  \sum_{\tau=0}^{\kappa-1} \left\Vert \bar{\nabla} f_u \left(\bar{\mathbf{w}}_u^{t,\tau} | \mathcal{D}_u^t \right) \right\Vert^2  + \nonumber\\
    &\Squad 2 \kappa \sum_{u=0}^{U-1} \alpha_u \left[ 1 - 2\mathfrak{q}_u^t (1 - \mathfrak{q}_u^t)\right] \sum_{\tau=0}^{\kappa-1} \left\Vert \bar{\nabla} f_u \left(\bar{\mathbf{w}}_u^{t,\tau} | \mathcal{D}_u^t \right) \right\Vert^2 \Bigg\} \nonumber\\
    &\overset{(g)}{\leq} \frac{\beta (\eta^t)^2} {2} \Bigg\{  \kappa \sigma^2 \sum_{u=0}^{U-1} \alpha_u^2 \left[2 + 2q +(4+q)\left(\mathfrak{q}_u^t\right)^2 - (3q+4)\mathfrak{q}_u^t \right] + q \kappa \sum_{u=0}^{U-1} \alpha_u^2  \sum_{\tau=0}^{\kappa-1} \left\Vert \bar{\nabla} f_u \left(\bar{\mathbf{w}}_u^{t,\tau} | \mathcal{D}_u^t \right) \right\Vert^2  + 2 \kappa \sum_{u=0}^{U-1} \alpha_u \sum_{\tau=0}^{\kappa-1} \left\Vert \bar{\nabla} f_u \left(\bar{\mathbf{w}}_u^{t,\tau} | \mathcal{D}_u^t \right) \right\Vert^2 \Bigg\} \nonumber\\
    &\overset{(h)}{\leq} \frac{\beta (\eta^t)^2} {2} \Bigg\{  \kappa \sigma^2 \sum_{u=0}^{U-1} \alpha_u^2 \left[2 + 2q +(4+q)\left(\mathfrak{q}_u^t\right)^2 - (3q+4)\mathfrak{q}_u^t \right] + q \kappa \sum_{u=0}^{U-1} \alpha_u  \sum_{\tau=0}^{\kappa-1} \left\Vert \bar{\nabla} f_u \left(\bar{\mathbf{w}}_u^{t,\tau} | \mathcal{D}_u^t \right) \right\Vert^2  + 2 \kappa \sum_{u=0}^{U-1} \alpha_u \sum_{\tau=0}^{\kappa-1} \left\Vert \bar{\nabla} f_u \left(\bar{\mathbf{w}}_u^{t,\tau} | \mathcal{D}_u^t \right) \right\Vert^2 \Bigg\} \nonumber\\
    &=\frac{\beta (\eta^t)^2} {2} \left[  \kappa \sigma^2 \sum_{u=0}^{U-1} \alpha_u^2 \left(2 + 2q +(4+q)\left(\mathfrak{q}_u^t\right)^2 - (3q+4)\mathfrak{q}_u^t \right) + \left(2+q\right) \kappa \sum_{u=0}^{U-1} \alpha_u  \sum_{\tau=0}^{\kappa-1} \left\Vert \bar{\nabla} f_u \left(\bar{\mathbf{w}}_u^{t,\tau} | \mathcal{D}_u^t \right) \right\Vert^2\right],
    \label{third_term_sim}
\end{align}
where $(a)$ comes from the definition of variance. 
In $(b)$, we use the fact that the cross product terms becomes zero when the expectation (w.r.t $\mathbf{\mathfrak{q}}^t$) is taken.
Besides, $(c)$ is true since $\Vert \sum_{i=0}^{I-1} \mathbf{a}_i \Vert^2=\Vert \sum_{i=0}^{I-1} 1 \cdot \mathbf{a}_i \Vert^2 \leq I \sum_{i=0}^{I-1} \Vert \mathbf{a}_i \Vert^2$ from Cauchy-Schwarz inequality, while $(d)$ appears from the bounded variance assumption of the stochastic quantizer in Assumption $4$.
In $(e)$, we use the definition of $\mathbf{d}_u^t$, while $(f)$ is true from the bounded variance assumption of the stochastic gradients. 
Furthermore, in $(g)$, we use the fact that $0 \leq \mathfrak{q}_u^t \leq 1$ implies $(1-\mathfrak{q}_u^t) \leq 1$ and thus, $\frac{1}{2} \leq \left(1 - 2\mathfrak{q}_u^t (1-\mathfrak{q}_u^t)\right) \leq 1$.
Finally, $(h)$ is true due to the fact that $\Vert \mathbf{a} \Vert \geq 0$ for any vector $\mathbf{a}$, $0\leq \alpha_u \leq 1$, and $\sum_{u=0}^{U-1}\alpha_u=1$, which implies $\sum_{u=0}^{U-1} \alpha_u^2 \Vert \mathbf{a} \Vert^2 \leq \sum_{u=0}^{U-1} \alpha_u \Vert \mathbf{a} \Vert^2$.

To that end, plugging (\ref{second_term_sim}) and (\ref{third_term_sim}) into 
(\ref{conv_eqn0_1}), we get
\begin{align}
    &\mathbb{E} \left[f \left( \mathbf{w}^{t+1} | \mathcal{D}^{t+1} \right) \right] 
    \leq \mathbb{E} \big[f \left( \mathbf{w}^t|\mathcal{D}^t \right)\big] + \frac{\eta^t \beta^2}{2} \sum_{u=0}^{U-1} \alpha_u \sum_{\tau=0}^{\kappa-1} \left\Vert \mathbf{w}^t - \Bar{\mathbf{w}}_u^{t, \tau} \right\Vert^2 - \frac{\eta^t\kappa}{2} \left\Vert \bar{\nabla} f \left( \mathbf{w}^t|\mathcal{D}^t \right) \right\Vert^2 - \frac{\eta^t}{2} \sum_{u=0}^{U-1} \alpha_u \sum_{\tau=0}^{\kappa-1} \left\Vert \bar{\nabla} f_u \left( \Bar{\mathbf{w}}_u^{t, \tau} | \mathcal{D}_u^t \right) \right\Vert^2 + \nonumber\\
    & \Squad \frac{\beta (\eta^t)^2} {2} \left[  \kappa \sigma^2 \sum_{u=0}^{U-1} \alpha_u^2 \left(2 + 2q +(4+q)\left(\mathfrak{q}_u^t\right)^2 - (3q+4)\mathfrak{q}_u^t \right) + \left(2+q\right) \kappa \sum_{u=0}^{U-1} \alpha_u  \sum_{\tau=0}^{\kappa-1} \left\Vert \bar{\nabla} f_u \left(\bar{\mathbf{w}}_u^{t,\tau} | \mathcal{D}_u^t \right) \right\Vert^2\right] \nonumber\\
    &= \mathbb{E} \big[f \left( \mathbf{w}^t|\mathcal{D}^t \right)\big] + \frac{\beta \kappa (\eta^t)^2 \sigma^2} {2} \sum_{u=0}^{U-1} \alpha_u^2 \left(2 + 2q +(4+q)\left(\mathfrak{q}_u^t\right)^2 - (3q+4)\mathfrak{q}_u^t \right) + \frac{\eta^t \beta^2}{2} \sum_{u=0}^{U-1} \alpha_u \sum_{\tau=0}^{\kappa-1} \left\Vert \mathbf{w}^t - \Bar{\mathbf{w}}_u^{t, \tau} \right\Vert^2 - \frac{\eta^t\kappa}{2} \left\Vert \bar{\nabla} f \left( \mathbf{w}^t|\mathcal{D}^t \right) \right\Vert^2 \nonumber \\
    &\Squad - \frac{\eta^t}{2} \left(1 - \beta \kappa \eta^t(2+q) \right)\sum_{u=0}^{U-1} \alpha_u \sum_{\tau=0}^{\kappa-1} \left\Vert \bar{\nabla} f_u \left( \Bar{\mathbf{w}}_u^{t, \tau} | \mathcal{D}_u^t \right) \right\Vert^2.
    \label{conv_eqn1_0}
\end{align}
When $\eta^t\leq \frac{1}{\beta \kappa(2+q)}$, we have $0 \leq (1 - \beta \eta^t \kappa (2+q)) \leq 1$.
As such, we drop the last term in (\ref{conv_eqn1_0}).
Then, after rearranging the terms in (\ref{conv_eqn1_0}), dividing both sides by $\frac{\eta^t\kappa}{2}$ and taking expectations on both sides, we get the following
\begin{align}
    \mathbb{E} \left[ \left\Vert \bar{\nabla} f \left( \mathbf{w}^{t} | \mathcal{D}^t \right) \right\Vert^2  \right]
    &\leq \frac{2 \left( \mathbb{E} \left[ f \left(\mathbf{w}^{t} | \mathcal{D}^t \right) \right] - \mathbb{E} \left[ f \left( \mathbf{w}^{t+1} | \mathcal{D}^{t+1} \right) \right] \right)} {\eta^t\kappa} + \beta \kappa \eta^t \sigma^2 \sum_{u=0}^{U-1} \alpha_u^2 \left(2 + 2q +(4+q)\left(\mathfrak{q}_u^t\right)^2 - (3q+4)\mathfrak{q}_u^t \right) + \nonumber \\
    & \frac{\beta^2}{\kappa} \sum_{u=0}^{U-1} \alpha_u\sum_{\tau=0}^{\kappa-1} \underbrace{\mathbb{E} \left[\left \Vert \mathbf{w}^{t} - \Bar{\mathbf{w}}_u^{t, \tau} \right\Vert^2 \right]}_{\mathrm{T}_3}. \label{conv_eqn1_2}
\end{align}

Now, we simplify the $\mathrm{T}_3$ term as follows
\begin{align}
\label{t3termSimplified}
    \mathrm{T}_3 &= \mathbb{E} \left[\left \Vert \mathbf{w}^t -  \Bar{\mathbf{w}}_u^{t, \tau} \right\Vert^2 \right] \nonumber\\
    &=\mathbb{E} \left[\left \Vert \mathbf{w}^t -\mathbf{w}_u^t + \mathbf{w}_u^t -  \Bar{\mathbf{w}}_u^{t, \tau} \right\Vert^2 \right] \nonumber\\
    &\overset{(a)}{=} \mathbb{E} \left[\left \Vert \mathbf{w}_u^t -  \Bar{\mathbf{w}}_u^{t, \tau} \right\Vert^2 \right] \nonumber\\
    &= \mathbb{E} \left[\left \Vert \mathbf{w}_u^t - \Bar{\mathbf{w}}_u^{t, 0} + \Tilde{\eta}^t\sum_{\tau'=0}^{\tau-1} \bar{g}_u \left( \Bar{\mathbf{w}}_u^{t, \tau'} | \mathcal{D}_u^t \right) \right\Vert^2 \right] \nonumber\\
    &\overset{(b)}{\leq} 2 \mathbb{E} \left[\left \Vert \mathbf{w}_u^t - \Bar{\mathbf{w}}_u^{t, 0} \right\Vert^2 \right] + 2 \mathbb{E} \left[\left \Vert \Tilde{\eta}^t\sum_{\tau'=0}^{\tau-1} \bar{g}_u \left( \Bar{\mathbf{w}}_u^{t, \tau'} | \mathcal{D}_u^t \right)\right\Vert^2 \right] \nonumber\\
    &\overset{(c)}{=} 2 \mathbb{E} \left[\left \Vert \mathbf{w}^t_u - \Bar{\mathbf{w}}_u^{t, 0} \right\Vert^2 \right] + 2 \left(\Tilde{\eta}^t\right)^2 \mathbb{E} \left[\left \Vert \sum_{\tau'=0}^{\tau-1} \left\{ \bar{g}_u \left( \Bar{\mathbf{w}}_u^{t, \tau'} | \mathcal{D}_u^t \right) - \mathbb{E}\left[ \bar{g}_u \left( \Bar{\mathbf{w}}_u^{t, \tau'} | \mathcal{D}_u^t \right) \right] \right\} \right\Vert^2 \right] + 2(\Tilde{\eta}^t)^2 \left( \mathbb{E} \left[ \sum_{\tau'=0}^{\tau-1} \bar{g}_u \left( \Bar{\mathbf{w}}_u^{t, \tau'} | \mathcal{D}_u^t \right) \right] \right)^2 \nonumber\\
    &\overset{(d)}{=} 2 \mathbb{E} \left[\left \Vert \mathbf{w}^t_u - \Bar{\mathbf{w}}_u^{t, 0} \right\Vert^2 \right] + 2 \left(\Tilde{\eta}^t\right)^2 \sum_{\tau'=0}^{\tau-1} \mathbb{E} \left[\left \Vert \bar{g}_u \left( \Bar{\mathbf{w}}_u^{t, \tau'} | \mathcal{D}_u^t \right) - \bar{\nabla} f_u \left( \Bar{\mathbf{w}}_u^{t, \tau'} | \mathcal{D}_u^t \right) \right\Vert^2 \right] + 2(\Tilde{\eta}^t)^2 \left\Vert \sum_{\tau'=0}^{\tau-1} \bar{\nabla} f_u \left( \Bar{\mathbf{w}}_u^{t, \tau'} | \mathcal{D}_u^t \right) \right\Vert^2 \nonumber\\
    &\overset{(e)}{\leq} 2 \mathbb{E} \left[\left \Vert \mathbf{w}_u^t - \Bar{\mathbf{w}}_u^{t, 0} \right\Vert^2 \right] + 2 \kappa \left(\Tilde{\eta}^t\right)^2 \sigma^2 + 2(\Tilde{\eta}^t)^2 \left\Vert \sum_{\tau'=0}^{\tau-1} \bar{\nabla} f_u \left( \Bar{\mathbf{w}}_u^{t, \tau'} | \mathcal{D}_u^t \right) \right\Vert^2 \nonumber\\
    &=2 \mathbb{E} \left[\left \Vert \mathbf{w}_u^t - \Bar{\mathbf{w}}_u^{t, 0} \right\Vert^2 \right] + 2 \kappa \left(\Tilde{\eta}^t\right)^2 \sigma^2 + 2(\Tilde{\eta}^t)^2 \Bigg\Vert \sum_{\tau'=0}^{\tau-1} \bigg[ \bar{\nabla} f_u \left( \Bar{\mathbf{w}}_u^{t, \tau'} | \mathcal{D}_u^t \right) - \bar{\nabla} f_u \left( \Bar{\mathbf{w}}_u^{t, \tau'} | \mathcal{D}_u^{t-1} \right) + \bar{\nabla} f_u \left( \Bar{\mathbf{w}}_u^{t, \tau'} | \mathcal{D}_u^{t-1} \right) - \bar{\nabla} f_u \left( \mathbf{w}^{t} | \mathcal{D}_u^{t-1} \right) + \nonumber\\
    &\Bquad \bar{\nabla} f_u \left( \mathbf{w}^{t} | \mathcal{D}_u^{t-1} \right) - \bar{\nabla} f_u \left( \mathbf{w}^{t} | \mathcal{D}_u^{t} \right) + \bar{\nabla} f_u \left( \mathbf{w}^{t} | \mathcal{D}_u^{t} \right) \bigg] \Bigg\Vert^2 \nonumber\\
    &\overset{(f)}{\leq} 2 \mathbb{E} \left[\left \Vert \mathbf{w}_u^t - \Bar{\mathbf{w}}_u^{t, 0} \right\Vert^2 \right] + 2 \kappa \left(\Tilde{\eta}^t\right)^2 \sigma^2 + 8(\Tilde{\eta}^t)^2 \left\Vert \sum_{\tau'=0}^{\tau-1} \left[ \bar{\nabla} f_u \left( \Bar{\mathbf{w}}_u^{t, \tau'} | \mathcal{D}_u^t \right) - \bar{\nabla} f_u \left( \Bar{\mathbf{w}}_u^{t, \tau'} | \mathcal{D}_u^{t-1} \right) \right] \right\Vert^2 + \nonumber\\
    &\Bquad 8(\Tilde{\eta}^t)^2 \left\Vert \sum_{\tau'=0}^{\tau-1} \left[\bar{\nabla} f_u \left( \Bar{\mathbf{w}}_u^{t, \tau'} | \mathcal{D}_u^{t-1} \right) - \bar{\nabla} f_u \left( \mathbf{w}^{t} | \mathcal{D}_u^{t-1} \right) \right]\right\Vert^2 + \nonumber\\
    &\Bquad 8(\Tilde{\eta}^t)^2 \left\Vert \sum_{\tau'=0}^{\tau-1} \left[ \bar{\nabla} f_u \left( \mathbf{w}^{t} | \mathcal{D}_u^{t-1} \right) - \bar{\nabla} f_u \left( \mathbf{w}^{t} | \mathcal{D}_u^{t} \right) \right]\right\Vert^2 + 8(\Tilde{\eta}^t)^2 \left\Vert \sum_{\tau'=0}^{\tau-1} \left[\bar{\nabla} f_u \left( \mathbf{w}^{t} | \mathcal{D}_u^{t} \right) \right] \right\Vert^2 \nonumber\\
    &\overset{(g)}{\leq} 2 \mathbb{E} \left[\left \Vert \mathbf{w}_u^t - \Bar{\mathbf{w}}_u^{t, 0} \right\Vert^2 \right] + 2 \kappa \left(\Tilde{\eta}^t\right)^2 \sigma^2 + 8 \kappa^2 (\Tilde{\eta}^t)^2 \left\Vert  \bar{\nabla} f_u \left( \Bar{\mathbf{w}}_u^{t,\tau} | \mathcal{D}_u^t \right) - \bar{\nabla} f_u \left( \Bar{\mathbf{w}}_u^{t, \tau} | \mathcal{D}_u^{t-1} \right) \right\Vert^2 + \nonumber\\
    &\Bquad 8 \kappa^2 (\Tilde{\eta}^t)^2 \left\Vert \bar{\nabla} f_u \left( \Bar{\mathbf{w}}_u^{t,\tau} | \mathcal{D}_u^{t-1} \right) - \bar{\nabla} f_u \left( \mathbf{w}^{t} | \mathcal{D}_u^{t-1} \right) \right\Vert^2 + \nonumber\\
    &\Bquad 8\kappa^2 (\Tilde{\eta}^t)^2 \left\Vert \bar{\nabla} f_u \left( \mathbf{w}^{t} | \mathcal{D}_u^{t-1} \right) - \bar{\nabla} f_u \left( \mathbf{w}^{t} | \mathcal{D}_u^{t} \right) \right\Vert^2 + 8\kappa^2 (\Tilde{\eta}^t)^2 \left\Vert \bar{\nabla} f_u \left( \mathbf{w}^{t} | \mathcal{D}_u^{t} \right) \right\Vert^2 \nonumber\\
    &\overset{(h)}{\leq} 2 \mathbb{E} \left[\left \Vert \mathbf{w}_u^t - \Bar{\mathbf{w}}_u^{t, 0} \right\Vert^2 \right] + 2 \kappa \left(\Tilde{\eta}^t\right)^2 \sigma^2 + 8 \Phi_u^t \kappa^2 (\Tilde{\eta}^t)^2 +  8 \beta^2 \kappa^2 (\Tilde{\eta}^t)^2 \left\Vert  \Bar{\mathbf{w}}_u^{t,\tau} - \mathbf{w}^{t} \right\Vert^2 + 8 \Phi_u^t \kappa^2 (\Tilde{\eta}^t)^2 + \nonumber\\
    & \Bquad 8 \kappa^2 (\Tilde{\eta}^t)^2 \left[\rho_1 \left\Vert \bar{\nabla} f (\mathbf{w}^t | \mathcal{D}^t \right\Vert^2 + \rho_2 \epsilon_u^t \right],
\end{align}
where $(a)$ stems from the fact that $\mathbf{w}_u^t \gets \mathbf{w}^t$ during the synchronization phase at the starting of each \ac{fl} round, 
$(c)$ from the definition of variance

Taking expectation on both sides of (\ref{t3termSimplified}), and rearranging the terms we get
\begin{align}
    \mathbb{E} \left[\left \Vert \mathbf{w}^t - \Bar{\mathbf{w}}_u^{t, \tau} \right\Vert^2 \right] 
    \leq \frac{2 \mathbb{E} \left[\left \Vert \mathbf{w}_u^t - \Bar{\mathbf{w}}_u^{t, 0} \right\Vert^2 \right] + 2 \kappa \left(\Tilde{\eta}^t\right)^2 \sigma^2 + 8 \Phi_u^t \kappa^2 (\Tilde{\eta}^t)^2 + 8 \Phi_u^t \kappa^2 (\Tilde{\eta}^t)^2 + 8 \kappa^2 (\Tilde{\eta}^t)^2 \left[\rho_1 \left\Vert \bar{\nabla} f (\mathbf{w}^t | \mathcal{D}^t \right\Vert^2 + \rho_2 \epsilon_u^t \right] }{1 - 8 \beta^2 (\Tilde{\eta}^t)^2\kappa^2},
\end{align}
When $\Tilde{\eta}^t < \frac{1}{2\sqrt{2}\beta\kappa}$, we have $0 < (1 - 8\beta^2(\Tilde{\eta}^t)^2\kappa^2) < 1$. 
As such, we write 
\begin{align}
    \mathrm{T}_3 = \mathbb{E} \left[\left \Vert \mathbf{w}^t - \Bar{\mathbf{w}}_u^{t, \tau} \right\Vert^2 \right] 
    \leq 2 \mathbb{E} \left[\left \Vert \mathbf{w}_u^t - \Bar{\mathbf{w}}_u^{t, 0} \right\Vert^2 \right] + 2 \kappa \left(\Tilde{\eta}^t\right)^2 \sigma^2 + 16 \Phi_u^t \kappa^2 (\Tilde{\eta}^t)^2 + 8 \kappa^2 (\Tilde{\eta}^t)^2 \rho_1 \mathbb{E} \left[\left\Vert \bar{\nabla} f (\mathbf{w}^t | \mathcal{D}^t \right\Vert^2\right] + 8 \kappa^2 (\Tilde{\eta}^t)^2 \rho_2 \epsilon_u^t.
\end{align}

Finally, plugging $\mathrm{T}_3$ into (\ref{conv_eqn1_2}), we get
\begin{align}
    &\mathbb{E} \left[ \left\Vert \bar{\nabla} f \left( \mathbf{w}^{t} | \mathcal{D}^t \right) \right\Vert^2  \right]
    \leq \frac{2 \left( \mathbb{E} \left[ f \left(\mathbf{w}^{t} | \mathcal{D}^t \right) \right] - \mathbb{E} \left[ f \left( \mathbf{w}^{t+1} | \mathcal{D}^{t+1} \right) \right] \right)} {\eta^t\kappa} + \beta \kappa \eta^t \sigma^2 \sum_{u=0}^{U-1} \alpha_u^2 \left(2 + 2q +(4+q)\left(\mathfrak{q}_u^t\right)^2 - (3q+4)\mathfrak{q}_u^t \right) + \nonumber\\
    &\Squad \frac{\beta^2}{\kappa} \sum_{u=0}^{U-1} \alpha_u\sum_{\tau=0}^{\kappa-1} \bigg\{ 2 \mathbb{E} \left[\left \Vert \mathbf{w}_u^t - \Bar{\mathbf{w}}_u^{t, 0} \right\Vert^2 \right] + 2 \kappa \left(\Tilde{\eta}^t\right)^2 \sigma^2 + 16 \Phi_u^t \kappa^2 (\Tilde{\eta}^t)^2 + 8 \kappa^2 (\Tilde{\eta}^t)^2 \rho_1 \mathbb{E} \left[\left\Vert \bar{\nabla} f (\mathbf{w}^t | \mathcal{D}^t \right\Vert^2\right] + 8 \kappa^2 (\Tilde{\eta}^t)^2 \rho_2 \epsilon_u^t \bigg\} \nonumber\\
    &= \frac{2 \left( \mathbb{E} \left[ f \left(\mathbf{w}^{t} | \mathcal{D}^t \right) \right] - \mathbb{E} \left[ f \left( \mathbf{w}^{t+1} | \mathcal{D}^{t+1} \right) \right] \right)} {\eta^t\kappa} + \beta \kappa \eta^t \sigma^2 \sum_{u=0}^{U-1} \alpha_u^2 \left(2 + 2q +(4+q)\left(\mathfrak{q}_u^t\right)^2 - (3q+4)\mathfrak{q}_u^t \right) + 2 \kappa \beta^2 \left(\Tilde{\eta}^t\right)^2 \sigma^2 + \nonumber\\
    &\Squad 2 \beta^2 \sum_{u=0}^{U-1} \alpha_u \mathbb{E} \left[\left \Vert \mathbf{w}_u^t - \Bar{\mathbf{w}}_u^{t, 0} \right\Vert^2 \right] + 16 \beta^2 \kappa^2 (\Tilde{\eta}^t)^2 \sum_{u=0}^{U-1} \alpha_u \Phi_u^t + 8 \beta^2 \kappa^2 (\Tilde{\eta}^t)^2 \rho_1 \mathbb{E} \left[\left\Vert \bar{\nabla} f (\mathbf{w}^t | \mathcal{D}^t \right\Vert^2\right] + 8 \beta^2 \kappa^2 (\Tilde{\eta}^t)^2 \rho_2 \sum_{u=0}^{U-1} \alpha_u \epsilon_u^t \nonumber\\
    &= \frac{2 \left( \mathbb{E} \left[ f \left(\mathbf{w}^{t} | \mathcal{D}^t \right) \right] - \mathbb{E} \left[ f \left( \mathbf{w}^{t+1} | \mathcal{D}^{t+1} \right) \right] \right)} {\eta^t\kappa} + \beta \sigma^2 \left( \kappa \eta^t \sum_{u=0}^{U-1} \alpha_u^2 \left(2 + 2q +(4+q)\left(\mathfrak{q}_u^t\right)^2 - (3q+4)\mathfrak{q}_u^t \right) + 2 \beta \kappa (\Tilde{\eta}^t)^2 \right) + \nonumber\\
    &\Squad 2 \beta^2 \sum_{u=0}^{U-1} \alpha_u \mathbb{E} \left[\left \Vert \mathbf{w}_u^t - \Bar{\mathbf{w}}_u^{t, 0} \right\Vert^2 \right] + 16 \beta^2 \kappa^2 (\Tilde{\eta}^t)^2 \sum_{u=0}^{U-1} \alpha_u \Phi_u^t + 8 \beta^2 \kappa^2 (\Tilde{\eta}^t)^2 \rho_2 \sum_{u=0}^{U-1} \alpha_u \epsilon_u^t + 8 \rho_1 \beta^2 \kappa^2 (\Tilde{\eta}^t)^2  \mathbb{E} \left[\left\Vert \bar{\nabla} f (\mathbf{w}^t | \mathcal{D}^t \right\Vert^2\right].
\end{align}

Rearranging the terms, we write
\begin{align}
    & \mathbb{E} \left[\left\Vert \bar{\nabla} f \left( \mathbf{w}^{t} | \mathcal{D}^t \right) \right\Vert^2 \right] 
    \leq \Bigg[ \frac{2 \left( \mathbb{E} \left[ f \left(\mathbf{w}^{t} | \mathcal{D}^t \right) \right] - \mathbb{E} \left[ f \left( \mathbf{w}^{t+1} | \mathcal{D}^{t+1} \right) \right] \right)} {\eta^t\kappa} + \beta \sigma^2 \left( \kappa \eta^t \sum_{u=0}^{U-1} \alpha_u^2 \mathrm{C}_u\left(q,\mathfrak{q}_u^t \right) + 2 \beta \kappa (\Tilde{\eta}^t)^2 \right) + \nonumber\\
    &\Squad 2 \beta^2 \sum_{u=0}^{U-1} \alpha_u \mathbb{E} \left[\left \Vert \mathbf{w}_u^t - \Bar{\mathbf{w}}_u^{t, 0} \right\Vert^2 \right] + 16 \beta^2 \kappa^2 (\Tilde{\eta}^t)^2 \sum_{u=0}^{U-1} \alpha_u \Phi_u^t + 8 \beta^2 \kappa^2 (\Tilde{\eta}^t)^2 \rho_2 \sum_{u=0}^{U-1} \alpha_u \epsilon_u^t \Bigg] \Big/ \big[1 - 8 \rho_1 \beta^2 (\Tilde{\eta}^t)^2\kappa^2\big]. \label{convBoundAvgOverTime_Eq2}
\end{align}
where $\mathrm{C}_u \left(q,\mathfrak{q}_u^t \right) \coloneqq \left(2 + 2q +(4+q)\left(\mathfrak{q}_u^t\right)^2 - (3q+4)\mathfrak{q}_u^t \right)$.

Similar to our previous assumption, if $\Tilde{\eta}^t< \frac{1}{2\sqrt{2\rho_1}\beta\kappa}$, we have $0 < \left( 1 - 8\rho_1\beta^2(\Tilde{\eta}^t)^2\kappa^2 \right) < 1$.
Therefore, we simplify (\ref{convBoundAvgOverTime_Eq2}) as
\begin{align}
    \mathbb{E} \left[\left\Vert \bar{\nabla} f \left( \mathbf{w}^{t} | \mathcal{D}^t \right) \right\Vert^2 \right] 
    &\leq \Bigg[ \frac{2 \left( \mathbb{E} \left[ f \left(\mathbf{w}^{t} | \mathcal{D}^t \right) \right] - \mathbb{E} \left[ f \left( \mathbf{w}^{t+1} | \mathcal{D}^{t+1} \right) \right] \right)} {\eta^t\kappa} + \beta \sigma^2 \left( \kappa \eta^t \sum_{u=0}^{U-1} \alpha_u^2 \mathrm{C}_u\left(q,\mathfrak{q}_u^t \right) + 2 \beta \kappa (\Tilde{\eta}^t)^2 \right) + \nonumber\\
    &~ 2 \beta^2 \sum_{u=0}^{U-1} \alpha_u \mathbb{E} \left[\left \Vert \mathbf{w}_u^t - \Bar{\mathbf{w}}_u^{t, 0} \right\Vert^2 \right] + 16 \beta^2 \kappa^2 (\Tilde{\eta}^t)^2 \sum_{u=0}^{U-1} \alpha_u \Phi_u^t + 8 \beta^2 \kappa^2 (\Tilde{\eta}^t)^2 \rho_2 \sum_{u=0}^{U-1} \alpha_u \epsilon_u^t \Bigg]. 
\end{align}

Finally, averaging over time results in the following
\begin{align}
    &\frac{1}{T}\sum_{t=0}^{T-1}  \mathbb{E} \left[\left\Vert \bar{\nabla} f \left( \mathbf{w}^{t} | \mathcal{D}^t \right) \right\Vert^2 \right]  
    \leq \frac{2}{\kappa T} \sum_{t=0}^{T-1} \frac{\left( \mathbb{E} \left[ f \left(\mathbf{w}^{t} | \mathcal{D}^t \right) \right] - \mathbb{E} \left[ f \left( \mathbf{w}^{t+1} | \mathcal{D}^{t+1} \right) \right] \right)} {\eta^t} + \frac{\beta  \sigma^2}{T} \left( \sum_{u=0}^{U-1} \alpha_u^2 \sum_{t=0}^{T-1} \eta^t  \mathrm{C}_u\left(q, \mathfrak{q}_u^t\right) + 2 \beta \kappa \sum_{t=0}^{T-1} (\Tilde{\eta}^t)^2 \right) + \nonumber\\
    &\Mquad \frac{16 \beta^2 \kappa^2}{T} \sum_{t=0}^{T-1} \left(\Tilde{\eta}^t\right)^2 \sum_{u=0}^{U-1} \alpha_u \Phi_u^t + \frac{8 \beta^2 \kappa^2 \rho_2}{T} \sum_{t=0}^{T-1} \left(\Tilde{\eta}^t\right)^2 \sum_{u=0}^{U-1} \alpha_u \epsilon_u^t + \frac{2 \beta^2}{T} \sum_{t=0}^{T-1} \sum_{u=0}^{U-1} \alpha_u \mathbb{E} \left[\left \Vert \mathbf{w}_u^t - \Bar{\mathbf{w}}_u^{t, 0} \right\Vert^2 \right], 
\end{align}
which concludes the proof.
\end{proof}

\end{appendices}

\end{document}